\PassOptionsToPackage{hyperfootnotes=false}{hyperref}
\documentclass[accepted]{uai2026}

\usepackage[american]{babel}

\usepackage{natbib}
\usepackage{mathtools}
\usepackage{booktabs}
\usepackage{tikz}
\usepackage[ruled,vlined]{algorithm2e}
\usepackage{aligned-overset}
\usepackage{amsfonts}
\usepackage{amsmath}
\usepackage{amssymb}
\usepackage{amsthm}
\usepackage{booktabs}
\usepackage{colonequals}
\usepackage{multirow}
\usepackage{url}
\usepackage{subcaption}
\usepackage{tcolorbox}

\title{Eigenvalue Calibration for Semantic Embeddings \\ of Large Language Models}

\newcommand{\blfootnote}[1]{%
  \begingroup
  \renewcommand\thefootnote{}\footnote{#1}%
  \addtocounter{footnote}{-1}%
  \endgroup
}
\author[1]{Sebastian G. Gruber $^{\dagger}$}
\author[2,3,4]{\href{mailto:<nassim.walha@dkfz-heidelberg.de>?Subject=Your UAI 2026 paper}{Nassim Walha $^{\dagger}$}}
\author[6]{Francis Bach}
\author[2,3,4,5]{Florian Buettner}

\affil[1]{%
    ESAT-PSI\\
    KU Leuven\\
    Belgium
}
\affil[2]{%
    German Cancer Research Center (DKFZ)\\
    Heidelberg\\
    Germany
}
\affil[3]{%
    German Cancer Consortium (DKTK)\\
    Germany
  }

\affil[4]{%
    Goethe University Frankfurt\\
    Germany
  }

\affil[5]{%
    Frankfurt Cancer Institute\\
    Germany
  }

\affil[6]{%
    PSL Research University / Inria \\
    France.
  }

\newtheorem{definition}{Definition}
\newtheorem{theorem}{Theorem}
\newtheorem{lemma}{Lemma}
\newtheorem{proposition}{Proposition}

\begin{document}

\maketitle
\blfootnote{$^{\dagger}$Equal contribution. Alphabetical order.}

\begin{abstract}
Uncertainty quantification is central to the reliable deployment of large language models (LLMs), and eigenvalues of semantic embeddings have recently emerged as a key tool in state-of-the-art methods.
However, conventional calibration results developed for classification probabilities cannot be directly transferred to eigenvalues. We address this gap by proposing a novel framework for calibrating the eigenvalues of semantic embeddings. We interpret LLMs combined with semantic embeddings of their generated answers as density matrix predictors, and we propose a novel approach to calibrate density matrix predictors by applying temperature scaling to their eigenvalues.
We establish entropy–risk equivalence under calibration, derive a central calibration inequality specific to eigenvalues, and prove that temperature-scaled eigenvalues optimize calibration when minimizing proper score risks.
Experiments on a variety of real-world settings show that current LLMs are systematically overconfident, and validate our theoretical findings.
Together, these results advance the foundations and practice of uncertainty quantification for semantic embeddings.
\end{abstract}

\section{Introduction}
\begin{figure}[htbp]
    \centering
    \begin{subfigure}{0.235\textwidth}
        \centering
        \includegraphics[width=0.9\linewidth]{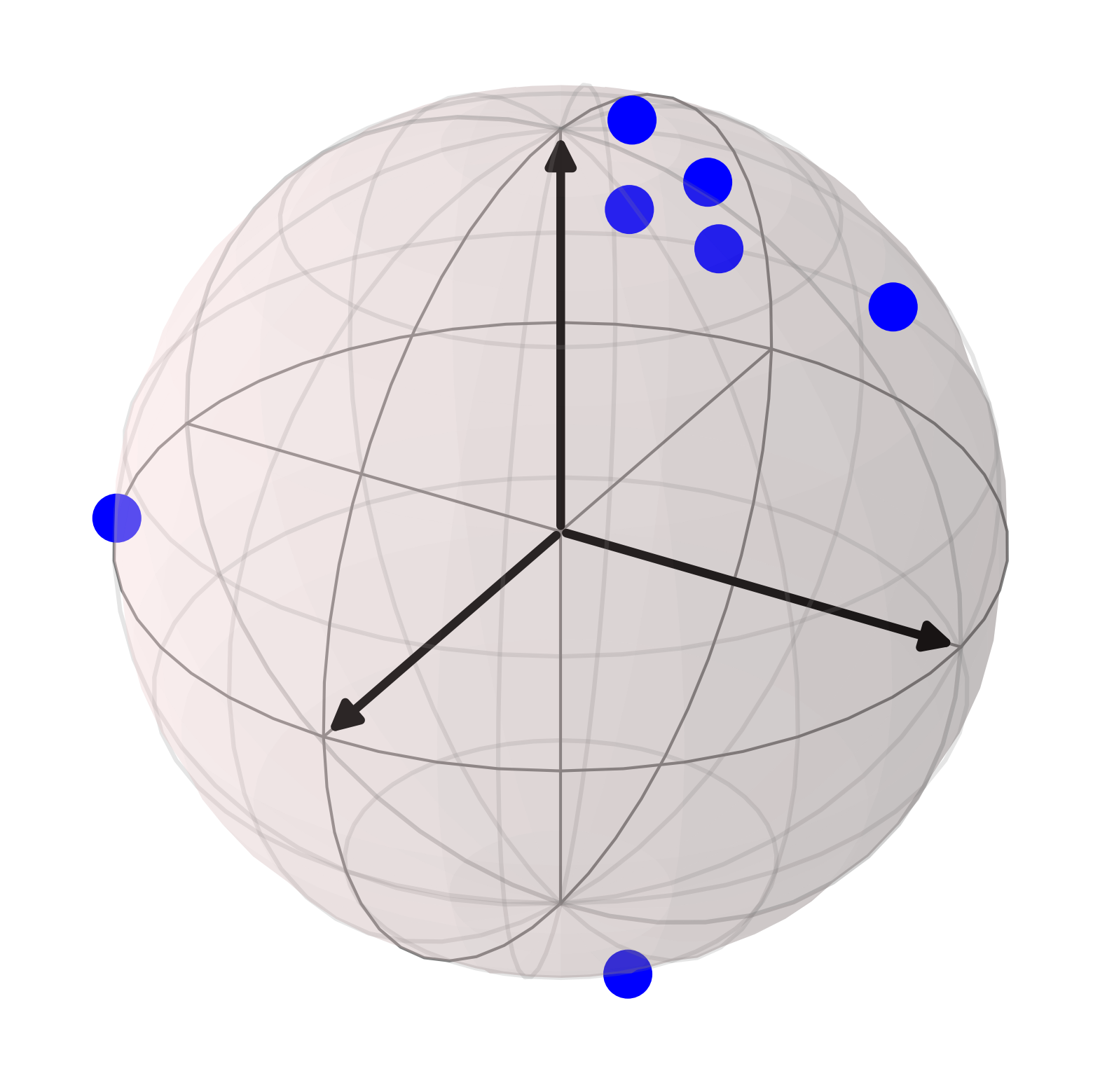}
        \caption{Semantic embeddings \\ \phantom{x}}
        \label{fig:entry_sub1}
    \end{subfigure}
    \begin{subfigure}{0.235\textwidth}
        \centering
        \includegraphics[width=\linewidth]{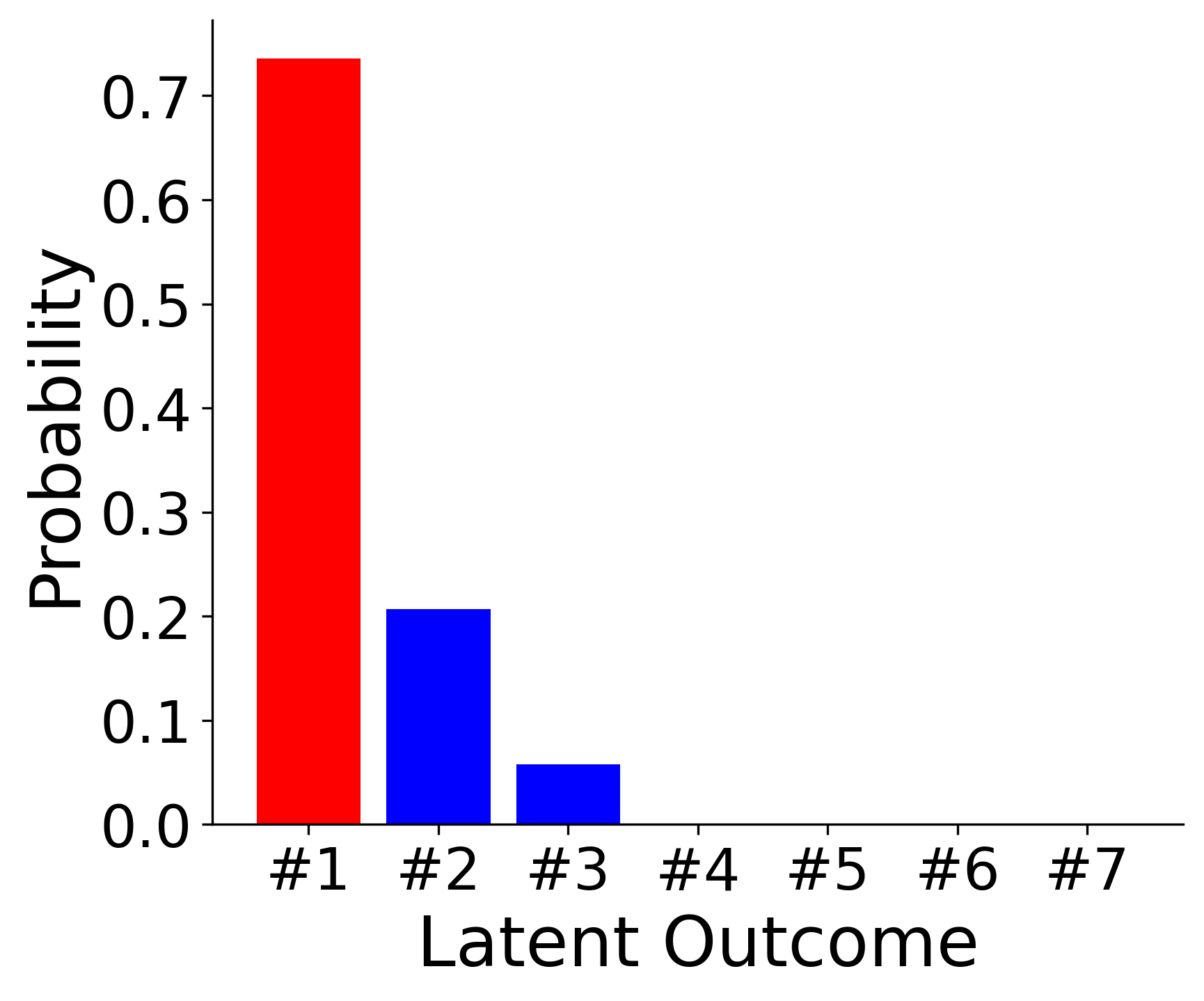}
        \caption{Eigenvalues as Probabilities}
        \label{fig:entry_sub2}
    \end{subfigure} \\
    \begin{subfigure}{0.235\textwidth}
        \centering
        \includegraphics[width=\linewidth]{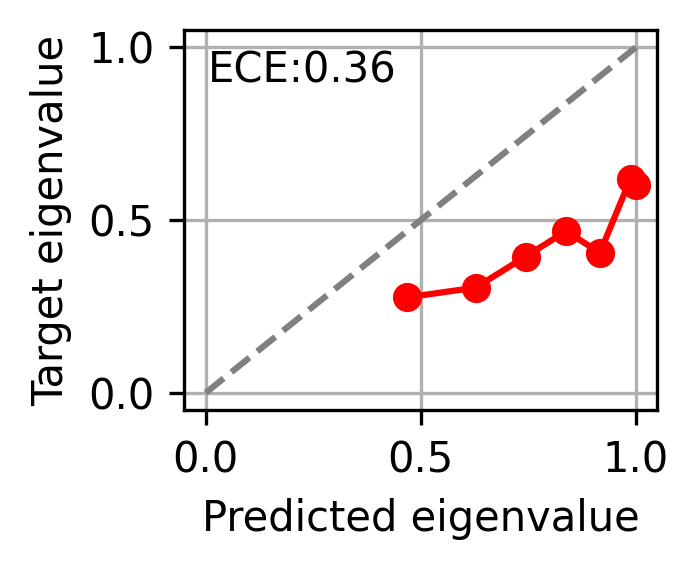}
        \caption{Before temperature scaling}
        \label{fig:entry_sub4}
    \end{subfigure}
    \begin{subfigure}{0.235\textwidth}
        \centering
        \includegraphics[width=\linewidth]{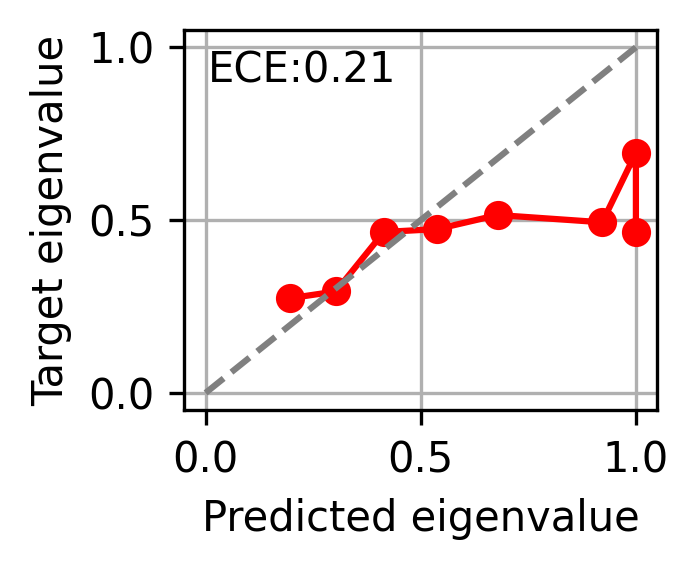}
        \caption{After temperature scaling}
        \label{fig:entry_sub5}
    \end{subfigure}
    
    \caption{Normalised semantic embeddings reside in a hypersphere (Figure~\ref{fig:entry_sub1}). We interpret the eigenvalues of the respective density matrix as probabilities of latent outcomes (Figure~\ref{fig:entry_sub2}).
    Large language models are overconfident in their predicted maximum eigenvalue (Figure~\ref{fig:entry_sub4}), which can be adjusted via temperature scaling (Figure~\ref{fig:entry_sub5}) resulting in a lower expected calibration error (ECE), and, thus, more reliable uncertainties.}
    \label{fig:overview}
\end{figure}

Uncertainty quantification has become a cornerstone for assessing the reliability of modern machine learning models, particularly large language models (LLMs) \citep{shorinwa2025survey}.
In high-stakes applications, calibrated uncertainty estimates are essential for downstream decision-making, model comparison, and human–AI collaboration \citep{silva2023classifier, maier2024metrics}.
Recent advances have shown that semantic embeddings of LLM outputs provide a powerful basis for uncertainty quantification, enabling fine-grained measures of predictive confidence \citep{gruber2024biasvariancecovariance, nikitin2024kernel, walha2025finegrained}.
In particular, the eigenvalues constructed from embeddings capture uncertainty through entropy quantities, and have already been adopted in state-of-the-art LLM uncertainty quantification methods \citep{nikitin2024kernel, walha2025finegrained}.

Despite this progress, a fundamental gap remains: conventional calibration results developed for classification probabilities cannot be directly transferred to eigenvalues. 
In classification, calibration aligns predicted probabilities with empirical frequencies, providing interpretable confidence estimates \citep{ANewVectorPartitionoftheProbabilityScore}.
In contrast, the eigenvalues of density matrices (constructed from embeddings) encode latent outcome probabilities in a different mathematical space, raising the question of how calibration should be defined, analyzed, and optimized in this setting.
Addressing this gap is essential to ensure that embedding-based uncertainty quantification methods are both theoretically principled and practically reliable.

In this paper, we provide the first comprehensive study of eigenvalue calibration for density matrix predictors in general and semantic embeddings of LLMs in particular.
To connect our theory with practice, we interpret LLMs as density matrix predictors and evaluate our framework extensively on real-world settings.
Figure~\ref{fig:overview} presents an overview of our setup, and a motivating example from our evaluations is given as follows.
For the TriviaQA question “What links do Bollywood, Hollywood and Lollywood have?”, the Phi-4 Mini model generates 20 candidate responses. Among these, 8 correspond to the correct answer “film-making industry”, while the remaining 12 incorrectly focus on “film-making in the Indian subcontinent.” Because this incorrect response is the most frequent one and semantically similar to the correct generations, the model exhibits overconfidence, with a maximum eigenvalue of the predicted density matrix $\lambda_{\max} = 0.81$. This is the predicted eigenvalue plotted on the x-axis of Figure~\ref{fig:entry_sub4} and Figure~\ref{fig:entry_sub5}. After applying temperature calibration with the optimal value $T = 2.51$, the predicted eigenvalue is reduced to $\lambda_{\max}=0.29$, mitigating overconfidence and more accurately reflecting the unreliability of the predicted answer.

In summary, our \textbf{contributions} are as follows:
\begin{itemize}
    \item We introduce a novel notion of calibration suitable for density matrix predictors, including eigenvalue-based uncertainty quantification via semantic embeddings in Section~\ref{sec:ev_calibration}.
    Specifically, calibration is required such that the average entropy predicts the model risk.
    \item We theoretically and empirically demonstrate how proper scores and temperature scaling can be used to optimise the calibration of eigenvalues, and show how current large language models are systematically overconfident in Section~\ref{sec:optim_calibration}.
    \item We propose a novel approach to plot reliability diagrams for transparent insights into the calibration of eigenvalues, which verifies our other results in Section~\ref{sec:rel_diagrams}.
\end{itemize}

\section{Background}

In this section, we give an overview of the background necessary to state our contributions.
Throughout this work, we assume a supervised learning setup, with input-target random variables $\left(X, Y \right)$ following a joint distribution $\mathbb{P}_{XY}$ over a sample space $\mathcal{X} \times \mathcal{Y}$.

First, we give an introduction to uncertainty calibration, followed by proper scores, density matrices,
and Bregman matrix divergences.

\paragraph{Calibration.}

Uncertainty calibration in classification is crucial for real-world applications \citep{silva2023classifier}.
We define a model $f \colon \mathcal{X} \to \mathcal{P}$ as \textbf{canonically calibrated} based on a joint distribution $\mathbb{P}_{XY}$ and a set $\mathcal{P}$ of potential target distributions if and only if 
\begin{equation}
\label{eq:canonical_cal}
    \mathbb{P} \left( Y \mid f \left( X \right) \right) \overset{a.s.}{=} f \left( X \right),
\end{equation}
where $a.s.$ refers to \emph{almost surely} \citep{vaicenavicius2019evaluating}.
Note that canonical calibration is also defined beyond classification under a more general set of distributions $\mathcal{P}$ and a general target space $\mathcal{Y}$ \citep{gruber2022better}.
In Section~\ref{sec:ev_calibration}, we introduce matrix and eigenvalue calibration arising from canonical calibration.

\paragraph{Proper Scores.}

A loss function of the form $S \colon \mathcal{P} \times \mathcal{Y} \to \mathbb{R} \cup \left\{ -\infty, \infty \right\}$ for which holds that 
\begin{equation}
    \mathbb{E}_{Y 
    \sim Q} \left[ S \left( P, Y \right) \right] \geq \mathbb{E}_{Y \sim Q} \left[ S \left( Q, Y \right) \right], \quad \forall P,Q \in \mathcal{P},
\end{equation}
with respect to a set of distributions $\mathcal{P}$ is referred to as proper scoring rule, or, in short, \textbf{proper score} \citep{gneitingscores}.
It is known that the respective \textbf{entropy function} $H_S \colon \mathcal{P} \to \mathbb{R}$ defined via $H_S \left( P \right) \coloneqq \mathbb{E}_{Y \sim P} \left[ S \left( P, Y \right) \right]$ is concave.
A common example is the log score, which leads to the Shannon entropy as entropy function \citep{gneitingscores}.

\paragraph{Density matrices and eigenvalues.}

We use density matrices (also referred to as density operators), which are positive semi-definite matrices with unit trace \citep{wilde2013quantum}.
A matrix $M$ is said to be positive semi-definite (p.s.d.) if and only if for all $a \in \mathbb{R}^d$ holds $a^\intercal M a \geq 0$.
Therefore, we define the space of all p.s.d. $d \times d$ matrices via $\mathbb{H}_d \coloneqq \left\{ M \in \mathbb{R}^{d \times d} \mid M \text{ p.s.d.} \right\}$.
The respective space of density matrices is then defined by $\mathbb{H}_d^\Delta \coloneqq \left\{ D \in \mathbb{H}_d \mid \operatorname{tr}D=1 \right\}$.
For any density matrix in $\mathbb{H}_d^\Delta$ holds that its eigenvalues are elements of the simplex $\Delta_d \coloneqq \left\{ p \in \mathbb{R}^d_{\geq 0} \mid \sum_i p_i = 1 \right\}$.
We also make use of spectral functions of the form $\mathbf{f} \colon \mathbb{H}_d \to \mathbb{R}^{d \times d}$, which are defined based on a function $f \colon \mathbb{R} \to \mathbb{R}$ via $\mathbf{f} \left( M \right) \coloneqq \sum_{i=1}^d f \left( \lambda_i \right) e_i e_i^\intercal$, where $\lambda_1, \dots \lambda_d$ and $e_1, \dots e_d$ are the eigenvalues and eigenvectors of $M$.
The maximum eigenvalue of a density matrix $D \in \mathbb{H}_d^\Delta$ is equal to its spectral norm $\left\lVert . \right\rVert_2$, i.e.,
$\lambda_{\max} \left( D \right) = \left\lVert D \right\rVert_2$ \citep{boyd2004convex}.
Thus, computing the maximum eigenvalue is a convex operation, which is relevant for Section~\ref{sec:ev_calibration}.

In classification, given a one-hot encoded random variable $Y_{\mathrm{oh}}$, we can recover its (marginal) distribution as a probability vector $\mathbb{P}_{Y_{\mathrm{oh}}}$ via
\begin{equation}
    \mathbb{P}_{Y_{\mathrm{oh}}} = \mathbb{E} \left[ Y_{\mathrm{oh}} \right] \in \Delta_d.
\end{equation}
Given another random variable $X$, its conditional distribution is recovered via
\begin{equation}
    \mathbb{P}_{Y_{\mathrm{oh}} \mid X} = \mathbb{E} \left[ Y_{\mathrm{oh}} \mid X \right],
\end{equation}
which maps into the simplex $\Delta_d$.
In a similar way, we can construct marginal and conditional density matrices based on random variables residing in a hypersphere \citep{wilde2013quantum}.
The \textbf{(marginal) density matrix} of a random variable $\mathbf{Y}$ with outcomes in the $d$-dimensional hypersphere $S_d \coloneqq \left\{ v \in \mathbb{R}^d \mid \left\lVert v \right\rVert_2 = 1 \right\}$ is given by
\begin{equation}
    \mathbb{D}_\mathbf{Y} \coloneqq \mathbb{E} \left[ \mathbf{Y} \mathbf{Y}^\intercal \right] \in \mathbb{H}_d^\Delta.
\end{equation}
Note that, generally, $\mathbb{D}_{\mathbf{Y}}$ does not characterize the distribution of $\mathbf{Y}$.
Further, its \textbf{conditional density matrix} with an additional random variable $X$ is defined as
\begin{equation}
    \mathbb{D}_{\mathbf{Y} \mid X} \coloneqq \mathbb{E} \left[ \mathbf{Y} \mathbf{Y}^\intercal \mid X \right],
\end{equation}
which maps into $\mathbb{H}_d^\Delta$.
Similar to probability vectors, it also holds for density matrices that $\mathbb{E} \left[ \mathbb{D}_{\mathbf{Y} \mid X} \right] = \mathbb{D}_\mathbf{Y}$.

Since the eigenvalues of a density matrix are non-negative and sum up to one, they can be interpreted as probabilities of non-observable (i.e. ``latent'') pure states in quantum mechanics \citep{wilde2013quantum}.
In the context of semantic embeddings, \emph{we propose to interpret eigenvalues as probabilities of semantic latent outcomes} (cf. Figure~\ref{fig:entry_sub1} and \ref{fig:entry_sub2}).
For example, ``It's Paris.'' and ``The answer is Paris.'' are literally different answers but they are semantically the same (latent) answer covered by the same eigenvalue.

\paragraph{Bregman matrix divergences.}

It is a known fact that proper scores induce Bregman divergences \citep{10.3150/16-BEJ857}.
In this work, we offer the first formal definition of proper scores defined for density matrices, which, as we show in a later section, induce Bregman matrix divergences as defined in \citep{kulis2009low}.
Given a convex and differentiable function $\phi \colon \mathbb{H}_d \to \mathbb{R}$, the generated \textbf{Bregman matrix divergence} $\operatorname{Div}_{\phi} \colon \mathbb{H}_d \times \mathbb{H}_d \to \mathbb{R}_{\geq 0}$ is given by
\begin{equation}
    \operatorname{Div}_{\phi} \left( x, y \right) \coloneqq \phi \left( x \right) - \phi \left( y \right) - \operatorname{tr} \left[ \nabla \phi \left( y \right) \left( x - y \right)^\intercal \right].
\end{equation}

We can now present the contributions of our work.
We start with a novel notion of calibration regarding density matrices and its eigenvalues.

\section{Eigenvalue Calibration}
\label{sec:ev_calibration}

To connect our theoretical contributions with practical applications, it is of central importance to identify LLMs in combination with semantic embeddings as density matrix predictors.
    
\begin{tcolorbox}
\paragraph{LLMs as density matrix predictors.}
    Let $\mathcal{X}_{\text{text}}$ be the space of language and $\mathcal{P}_{\text{text}}$ be a set of distributions with support $\mathcal{X}_{\text{text}}$.
    Let $f \colon \mathcal{X}_{\text{text}} \to \mathcal{P}_{\text{text}}$ be an LLM, which allows to sample answers $a \sim f \left( x \right)$ given an input query $x \in \mathcal{X}_{\text{text}}$, and let $e \colon \mathcal{X}_{\text{text}} \to \mathbb{R}^d$ be a semantic embedding model.
    From here on, we use the associated density matrix predictor $d \colon \mathcal{X}_{\text{text}} \to \mathbb{H}_d^\Delta$ given by
    \begin{equation}
        d \left( x \right) \coloneqq \mathbb{E}_{a \sim f \left( x \right)} \left[ e \left( a \right)e \left( a \right)^\intercal \right].
    \end{equation}
\end{tcolorbox}
In essence, the density matrix predictor is simply the expected outer product of the answers' semantic embeddings.
In practice, the predictor $d$ can be approximated via the estimator $\hat{d} \left( x \right) \coloneqq \frac{1}{m} \sum_{i=1}^m e \left( a_i \right)e \left( a_i \right)^\intercal$ given sampled answers $a_1, \dots, a_m \sim f \left( x \right)$.

Similarly to Figure~\ref{fig:entry_sub2}, \citet{walha2025finegrained} interpret the eigenvalues of $\hat{d}$ as probabilities over different latent semantic outcomes within the sampled answers. In the following, we motivate this interpretation: The first eigenvector of $\hat{d}$ is the 1-D semantic direction that minimizes residual variance when the embedded answers are projected onto it, i.e., it represents the dominant latent semantic direction. The corresponding eigenvalue measures the squared aggregate mass of the answers along that direction, so it is naturally interpretable as confidence concentrated on a latent semantic mode. Further eigenvectors are built orthogonally to previous ones, which matches the orthogonality of latent semantic directions. Two extremes make this concrete: if all sampled answers are semantically identical, then $\lambda_{\max} =1$ and all other eigenvalues are $0$, which maximizes confidence; if they are mutually orthogonal, then all non-zero eigenvalues are equal with $\lambda_{\max} = \lambda_1 = \lambda_2 = \ldots = \lambda_m = \frac{1}{m}$, thus maximizing entropy.
\begin{definition}[Matrix and Eigenvalue Calibration]
We say a density matrix predictor $d$ is \textbf{matrix calibrated} w.r.t. the joint distribution $\mathbb{P}_{XY}$ if
\begin{equation}
\label{eq:psd_matrix_cal}
    \mathbb{D}_{Y \mid d \left( X \right)} \overset{a.s.}{=} d \left( X \right),
\end{equation}    
and it is \textbf{eigenvalue calibrated} if
\begin{equation}
\label{eq:EV_cal}
    \lambda_{\max} \left(  \mathbb{D}_{Y \mid \Lambda_X} \right) \overset{a.s.}{=} \Lambda_X,
\end{equation}
with $\Lambda_X \coloneqq \lambda_{\max} \left( d \left( X \right) \right)$.
\end{definition}

In other words, similar to when we say a model is canonically calibrated when its predicted distribution matches the target distribution (given the prediction), we say a density matrix predictor is matrix calibrated if its predicted matrix matches the target density matrix.
For eigenvalue calibration, we condition on the predicted maximum eigenvalue.
However, even though eigenvalue calibration seems simpler than matrix calibration, we argue based on our following results that matrix calibration is more meaningful in practice.

Following the previous definitions, we can create a relationship between the different calibration notions.

\begin{theorem}[Eigenvalue Calibration Inequality]
\label{th:cal_inequality}
    Let $f$ be an LLM and $d$ its respective density matrix predictor.
    If $f$ is canonically calibrated then $d$ is matrix calibrated.
    Further, if $d$ is matrix calibrated, then
    \begin{equation}
    \label{eq:EV_semi_cal}
        \lambda_{\max} \left(  \mathbb{D}_{Y \mid d \left( X \right)} \right) \overset{a.s.}{=} \Lambda_X,
    \end{equation}
    with $\Lambda_X \coloneqq \lambda_{\max} \left( d \left( X \right) \right)$, but also
    \begin{equation}
    \label{eq:EV_uncal}
        \lambda_{\max} \left(  \mathbb{D}_{Y \mid \Lambda_X} \right) \overset{a.s.}{\leq} \Lambda_X.
    \end{equation}
\end{theorem}

The proof is located in Appendix~\ref{app:proofs} and makes use of \citet[Theorem 1]{gruber2024novel} in combination with the convexity of the spectral norm.
Theorem~\ref{th:cal_inequality} clarifies that matrix calibration follows from canonical calibration, similar to other calibration notions in classification \citep{gruber2022better, gupta2022top}.
However, it also offers the surprising fact that eigenvalue calibration does not follow, which impacts our ability to compute reliability diagrams for calibration diagnostics.
Therefore, in a later section, we propose to rather use Equation~\ref{eq:EV_semi_cal} to compute meaningful reliability diagrams.

We now introduce proper matrix scores, which connect matrix calibration with risk minimization.

\begin{definition}
A function $\mathbf{S} \colon \mathbb{H}_d^\Delta \to \mathbb{R}^{d\times d}$ is defined to be a \textbf{proper matrix score} if
\begin{equation}
    \operatorname{tr} \left[ \mathbf{S} \left( M_{\mathrm{pred}} \right) M_{\mathrm{target}} \right] \geq \operatorname{tr} \left[ \mathbf{S} \left( M_{\mathrm{target}} \right) M_{\mathrm{target}} \right]
\end{equation}
for all $M_{\mathrm{pred}}, M_{\mathrm{target}} \in \mathbb{H}_d^\Delta$.
\end{definition}

Our definition of proper matrix scores is a subclass of proper scores, which is usable for semantic embeddings, as the following shows.

\begin{proposition}
\label{prop:mats=ps}
    A proper matrix score $\mathbf{S} \colon \mathbb{H}_d^\Delta \to \mathbb{R}^{d\times d}$ generates a proper score $S \colon \mathcal{P}_d \times \mathcal{Y}_d \to \mathbb{R}$ with respect to the set of distributions $\mathcal{P}_d \coloneqq \left\{ P \mid \int_{S_d} \left\lVert x \right\rVert_2^2 \mathrm{d} P \left( x \right) < \infty\right\}$ defined on the hypersphere and $\mathcal{Y}_d = \mathbb{R}^d$.
    The respective proper score is given by
    \begin{equation}
        S \left( P, y \right) \coloneqq y^\intercal \mathbf{S} \left( \mathbb{E}_{x \sim P} \left[ xx^\intercal \right] \right) y.
    \end{equation}
\end{proposition}

The proof is presented in Appendix~\ref{app:proofs}.

The respective \textbf{matrix entropy function} $H_{\mathbf{S}} \colon \mathbb{H}_d^\Delta \to \mathbb{R}$ is given by
\begin{equation}
    H_{\mathbf{S}} \left( M \right) \coloneqq \operatorname{tr} \left[ S \left( M \right) M \right].
\end{equation}
Further, we can define the associated \textbf{divergence function} $D_{\mathbf{S}} \colon \mathbb{H}_d^\Delta \times \mathbb{H}_d^\Delta \to \mathbb{R}_{\geq 0}$ via
\begin{equation}
    D_{\mathbf{S}} \left( M_1, M_2 \right) \coloneqq \operatorname{tr} \left[ \mathbf{S} \left( M_2 \right) M_1 - \mathbf{S} \left( M_1 \right) M_1 \right].
\end{equation}
We offer the following statements regarding associated entropy and divergence functions.
\begin{proposition}
\label{prop:ent_to_div}
    Given a proper matrix score $\mathbf{S}$, it holds
    that $H_{\mathbf{S}}$ is concave.
    Further, if $H_{\mathbf{S}}$ is also differentiable,
    then
    \begin{equation}
        D_{\mathbf{S}} = \operatorname{Div}_{-H_{\mathbf{S}}}.
    \end{equation}
\end{proposition}

This follows from the fact that a proper matrix score is a proper score, for which \cite{10.3150/16-BEJ857} shows similar results, and since a density matrix is constructed from a distribution in a linear manner (cf. Appendix~\ref{app:proofs}).
This result connects our definition of proper matrix scores to the already existing definition of Bregman matrix divergences.

Further, since we assume a supervised learning setup with joint distribution $\mathbb{P}_{X\mathbf{Y}}$ and density matrix predictor $d$, we are interested in the respective \textbf{risk} given a proper matrix score $\mathbf{S}$, which we define via
\begin{equation}
    \mathcal{R}_{\mathbf{S}} \left( d \right) \coloneqq \mathbb{E} \left[ \mathbf{Y}^\intercal \mathbf{S} \left( d \left( X \right) \right) \mathbf{Y} \right].
\end{equation}

The special case of a proper matrix score, which we use in our experiments, is the matrix version of the log score, also known as cross entropy loss.
The matrix log score is given by
$\mathbf{S}_{\log} \left( M \right) \coloneqq - \log \left( M \right)$, where $\log$ is used as a spectral function.
It follows that the associated entropy function is $H_{\mathbf{S}_{\log}} \left( M \right) = - \sum_{i=1}^d \lambda_i \log \lambda_i$, where $\lambda_1, \dots \lambda_d$ are the eigenvalues of $M \in \mathbb{H}_d^\Delta$.
This function is also known as von Neumann entropy \citep{wilde2013quantum}, and is essentially the Shannon entropy of the respective eigenvalues.
It is used for state-of-the-art results in LLM uncertainty quantification \citep{nikitin2024kernel}.
The related divergence function is given by $D_{\mathbf{S}_{\log}} \left( M_1, M_2 \right) = \operatorname{tr} \left[ \log \left( M_1 \right) M_1 - \log \left( M_2 \right) M_1 \right]$, which recovers the quantum relative entropy \citep{wilde2013quantum}.
The matrix log score will be used in our experiments due to its prominence.

We now offer a central result, which connects risk and entropy under matrix calibration.

\begin{lemma}
\label{prop:ev_risk=ev_unc}
Given a proper matrix score $\mathbf{S}$ and a density matrix predictor $d$, which is matrix calibrated for a joint distribution $\mathbb{P}_{X\mathbf{Y}}$, it holds that
\begin{equation}
\begin{split}
    \underbrace{\mathbb{E} \left[ H_{\mathbf{S}} \left( d \left( X \right) \right) \right]}_{\text{Expected Uncertainty}} = \underbrace{\mathcal{R}_{\mathbf{S}} \!\left( d \right)}_{\text{Risk}}.    
\end{split}
\end{equation}
\end{lemma}

The proof is located in Appendix~\ref{app:proofs}.
Lemma~\ref{prop:ev_risk=ev_unc} shows the importance of matrix calibration for uncertainty quantification.
It indicates that under calibration, the expected uncertainty given by the entropy function is an accurate representation of the respective risk.
However, exact calibration is infeasible in practice.
Therefore, we offer the following result, which proves that Lemma~\ref{prop:ev_risk=ev_unc} may also be achieved in an approximate manner.

\begin{theorem}[Entropy-Risk Convergence]
\label{th:risk_unc_convergence}
Let $d_1, d_2, \dots$ be a sequence of density matrix predictors such that $\lim_{n \to \infty} d_n$ is matrix calibrated.
Then, for a continuous proper matrix score $\mathbf{S}$, it holds
\begin{equation}
    \underbrace{\mathbb{E} \left[ H_{\mathbf{S}} \left( d_n \left( X \right) \right) \right]}_{\text{Expected Uncertainty}} - \underbrace{\mathcal{R}_{\mathbf{S}} \!\left( d_n \right)}_{\text{Risk}} \overset{}{\longrightarrow} 0
\end{equation}
for $n \to \infty$.
\end{theorem}

We have now established matrix calibration and its relevance for uncertainty quantification via semantic embeddings.
Next, we are proposing methodology, which allows to optimize matrix calibration via risk minimization in a post-hoc manner, similar to well-established procedures in classification \citep{gruber2022better}.

\section{Optimising Matrix Calibration}
\label{sec:optim_calibration}

\begin{figure}
    \centering
    \includegraphics[width=0.8\linewidth]{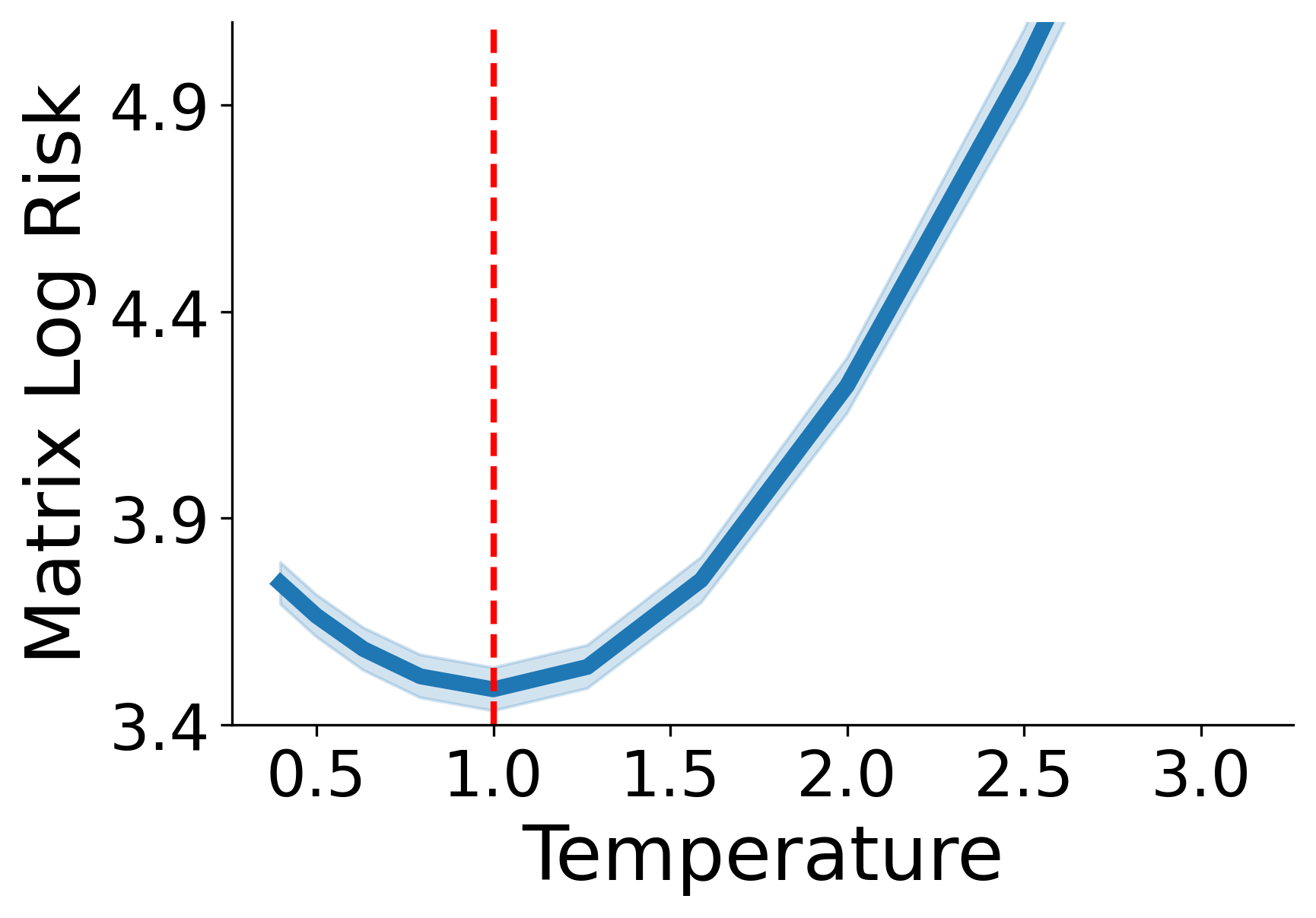}
    \caption{The matrix version of the cross entropy risk correctly determines the ground truth temperature (red line) for $n=100$ real-world semantic embeddings.}
    \label{fig:TS_risk_gt_dummy}
\end{figure}

In classification, it is a common procedure to transform a classifier's predicted probabilities via a simple transformation to improve its calibration \citep{guo2017calibration}.
If the transformation is injective and the optimization objective is a proper score-based risk, then a respective calibration error is provably optimized \citep{gruber2022better}.
However, this result cannot be readily transferred to our setup since the injectivity needs to hold from distribution space to distribution space.
Transforming distributions into density matrices is not injective in general.
This requires us to propose an analogous result, which changes the perspective from distributions towards density matrices.

To achieve this, we first translate the essential calibration-sharpness decomposition of proper scores introduced by \citet{Br_cker_2009} to proper matrix scores.

\begin{lemma}[Calibration-Sharpness Decomposition]
\label{th:ms_decomp}
Let $\mathbf{S}$ be a proper matrix score and $d$ be a density matrix predictor for a joint distribution $\mathbb{P}_{X\mathbf{Y}}$.
Assuming all integrals are finite, it holds
\begin{equation}
\begin{split}
    \underbrace{\mathcal{R}_S \! \left( d \right)}_{\text{Risk}} & = \underbrace{\operatorname{Cal}_S \!\left( d \right)}_{\text{Calibration}} - \underbrace{\mathbb{E} \left[ D_{\mathbf{S}} \left( \mathbb{D}_{Y \mid d\left(X\right)}, \mathbb{D}_{Y} \right) \right]}_{\text{Sharpness}} + \underbrace{H_{\mathbf{S}} \left( \mathbb{D}_{Y} \right)}_{\text{Noise}},
\end{split}
\end{equation}
with the respective calibration error being defined by
\begin{equation}
    \operatorname{Cal}_S \!\left( d \right) \coloneqq \mathbb{E} \left[ D_{\mathbf{S}} \left( \mathbb{D}_{Y \mid d\left(X\right)}, d\left(X\right) \right) \right].
\end{equation}
\end{lemma}

Since the calibration error $\operatorname{Cal}_S$ is induced by a proper score, we follow \citet{gruber2022better} and refer to it as proper calibration error (cf. Appendix~\ref{app:proofs} for details and proof).

\begin{theorem}[Calibration Optimization via Risk]
\label{th:cal_risk_optim}
    Given a proper matrix score $S \colon \mathbb{H}_d^\Delta \times \mathcal{Y}^d \to $, a density matrix predictor $d \colon \mathcal{X} \to \mathbb{H}_d^\Delta$ and an injective function $h \colon \mathbb{H}_d^\Delta \to \mathbb{H}_d^\Delta$, it holds
    \begin{equation}
    \begin{split}
        \underbrace{\mathcal{R}_S \!\left( h \!\circ\! d \right) - \mathcal{R}_S \! \left( d \right)}_{\text{Change in Risk}} = \underbrace{\operatorname{Cal}_S \!\left( h \!\circ\! d \right) - \operatorname{Cal}_S \!\left( d \right)}_{\text{Change in Calibration}}.     
    \end{split}
    \end{equation}
\end{theorem}
\citet{gruber2022better} show that an injective transformation of a model output changes exclusively the calibration term of a proper score.
Theorem~\ref{th:cal_risk_optim} translates this result to density matrices.

\begin{proposition}[Matrix Temperature Scaling]
\label{prop:mat_TS}
    The function $h_{\operatorname{TS}} \colon \mathbb{H}_d^\Delta \to \mathbb{H}_d^\Delta$ defined via 
    \begin{equation}
        h_{\operatorname{TS}} \left( M \right) \coloneqq \frac{1}{\operatorname{tr}M^{\alpha}}M^{\alpha} =  \frac{1}{\sum_{i=1}^d \lambda_i^\alpha} \sum_{i=1}^d \lambda_i^\alpha e_i e_i^\intercal
    \end{equation}
    is injective for $\alpha >0$.
\end{proposition}

The proof for injectivity is given in Appendix~\ref{app:proofs}.
The exponent $\alpha$ is used as a spectral function, which makes it act on the eigenvalues similar to how conventional temperature scaling transforms probability vectors \citep{guo2017calibration}.
We refer to $\frac{1}{\alpha}$ as the temperature (parameter) based on convention in the literature \citep{guo2017calibration, kull2019beyond, tomani2021parameterized}.

\begin{figure*}[htbp]
    \centering
    \includegraphics[width=1.\linewidth]{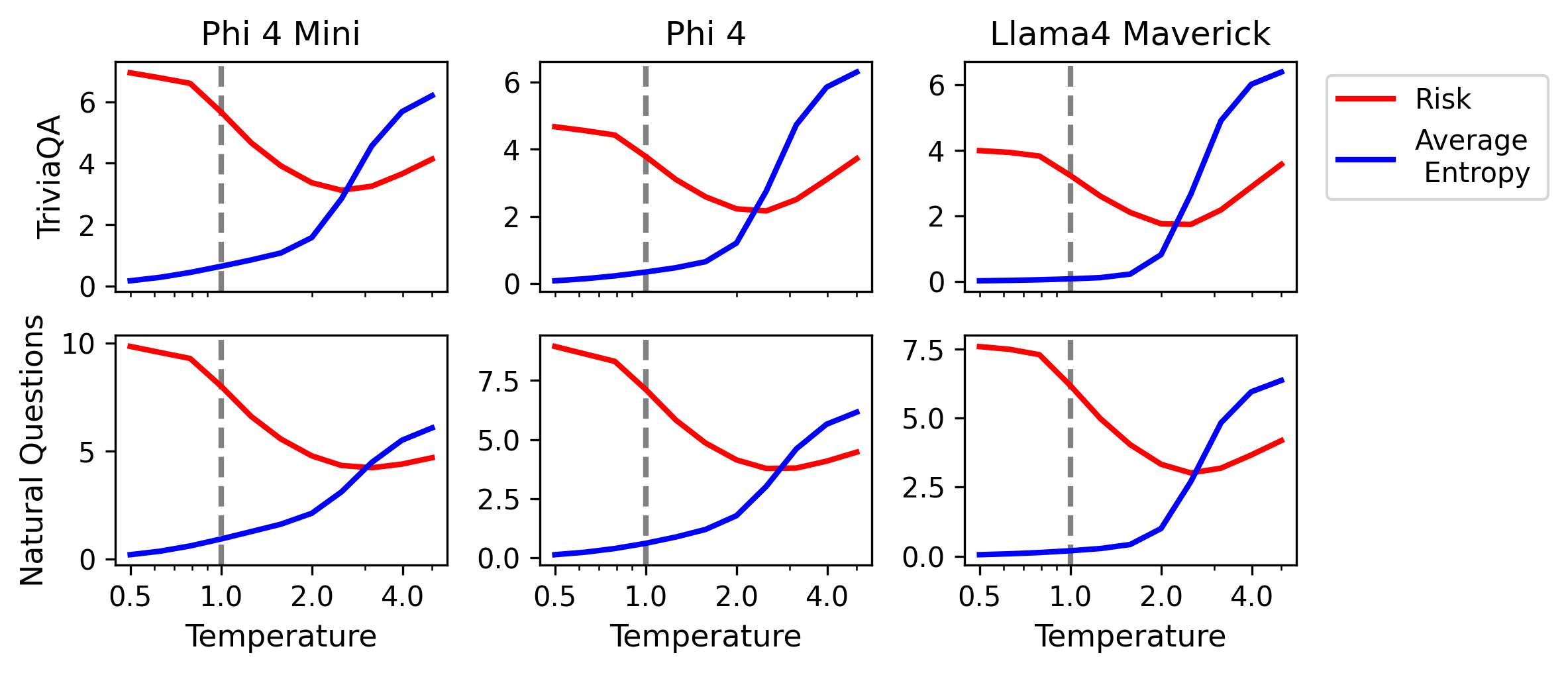}
    \caption{According to Theorem~\ref{th:cal_risk_optim}, the optimal temperature for calibration is indicated via the minimum risk.
    Here, the average entropy matches the empirical risk at the optimal temperature across all settings.
    This indicates that temperature scaling is sufficient to yield well-calibrated eigenvalues such that Theorem~\ref{th:risk_unc_convergence} holds.
    Further, in all settings, the optimal temperature is greater than one (grey dashed line).
    Thus, all models are systematically overconfident and require adjustment.}
    \label{fig:main_figure}
\end{figure*}

\begin{table*}[h]
\caption{AUROC values ($\uparrow$) based on fuzzy answer correctness with maximum eigenvalue and entropy as uncertainty scores before and after temperature scaling. Using the optimal temperature based on risk minimization improves the AUROC in most of the cases.}
\label{tbl:aurocs}
\begin{center}
\begin{tabular}{llllll}
\toprule
 &  & Eigenvalue & Eigenvalue TS & Entropy & Entropy TS \\
Dataset & Model &  &  &  &  \\
\midrule
\multirow[t]{3}{*}{Natural Questions} & Llama4 Maverick & 0.66 $\pm$ 0.001 & \textbf{0.67} $\pm$ 0.001 & \textbf{0.668} $\pm$ 0.001 & 0.663 $\pm$ 0.001 \\
 & Phi 4 & 0.742 $\pm$ 0.001 & \textbf{0.794} $\pm$ 0.0 & 0.764 $\pm$ 0.001 & \textbf{0.795} $\pm$ 0.001 \\
 & Phi 4 Mini & 0.768 $\pm$ 0.001 & \textbf{0.778} $\pm$ 0.001 & \textbf{0.776} $\pm$ 0.001 & 0.774 $\pm$ 0.001 \\
\cline{1-6}
\multirow[t]{3}{*}{TriviaQA} & Llama4 Maverick & \textbf{0.69} $\pm$ 0.001 & 0.688 $\pm$ 0.001 & 0.684 $\pm$ 0.001 & \textbf{0.685} $\pm$ 0.001 \\
 & Phi 4 & 0.814 $\pm$ 0.001 & \textbf{0.82} $\pm$ 0.001 & 0.818 $\pm$ 0.001 & \textbf{0.819} $\pm$ 0.001 \\
 & Phi 4 Mini & 0.808 $\pm$ 0.0 & \textbf{0.814} $\pm$ 0.001 & \textbf{0.814} $\pm$ 0.0 & 0.805 $\pm$ 0.001 \\
\cline{1-6}
\bottomrule
\end{tabular}
\end{center}
\end{table*}

\paragraph{Experiments.}

To verify our theory in practice, we conduct experiments for real-world settings of modern open source LLMs and question answering tasks.
All source code for our experiments is available at \url{https://github.com/MLO-lab/matrix_eigenvalue_calibration}.

We conduct experiments with the TriviaQA \citep{joshi2017triviaqa} and Natural Questions \citep{kwiatkowski2019natural} question answering datasets.
As answering LLMs, we use the open source models Phi 4 \citep{abdin2024phi}, Phi 4 Mini \citep{abouelenin2025phi}, and Llama4 Maverick \citep{meta2025llama4}.
As semantic embedding model, we use all-mpnet-base-v2 \citep{song2020mpnet}.

For each dataset, we randomly sample subsets from the validation split to conduct our experiments. A development subset of 300 examples per dataset is reserved for temperature scaling, where the temperature parameter is optimized via matrix log risk (see Figure \ref{fig:main_figure}).
Once the optimal temperature is determined, evaluation is performed on separate subsets of approximately 1700 examples each.
These evaluation subsets are used to construct reliability diagrams and compute the AUROC scores.
For each input $x$, we generate $m=20$ candidate responses from the LLM and embed them to estimate the corresponding density matrix $d(x)$.
Following prior work on LLM uncertainty \citep{kuhn2022semantic, walha2025finegrained}, responses are sampled at temperature 0.5 to balance diversity and accuracy.
More details are given in Appendix~\ref{app:exp_details}.

First, we verify that the matrix version of the cross entropy loss $\mathbf{S}_{\log}$ does indeed find the optimal temperature.
For this, we use $n=100$ random samples from the TriviaQA dataset and construct a constant predictor with $\frac{1}{n} \sum_{i=1}^n \mathbf{Y}_i \mathbf{Y}_i^\intercal$ from the target semantic embeddings.
A change in temperature for this predictor should increase the risk since it is constructed directly from the target samples.
As can be seen in Figure~\ref{fig:TS_risk_gt_dummy}, this is indeed the case, which confirms that temperatures too high and too low are indicated via an increased empirical risk.

Next, we conduct temperature scaling on the predicted density matrices of the LLMs.
We compare the empirical risk of the log score with the average entropy across a variety of temperature parameters.
The results across all settings are shown in Figure~\ref{fig:main_figure}.

First, it is predicted by Theorem~\ref{th:risk_unc_convergence} that the expected uncertainty quantified via the entropy is a predictor of the model risk, if the model converges towards matrix calibration.
This prediction is precisely observed in all settings since the line of the average entropy and the risk touch at the minimum risk.
Further, Theorem~\ref{th:cal_risk_optim} states that the temperature with the smallest risk is also the temperature for optimal calibration.
Therefore, Figure~\ref{fig:main_figure} shows that both of these theoretical findings support each other in practice and empirically verifies our theory in a practical setting.

The second finding is that the risk-optimal temperatures are all above one.
This concludes that the models are systematically overconfident across all settings.
Therefore, it is likely that overconfidence is a common occurrence in modern LLMs, which requires adjustment for reliable uncertainties, especially if we want the average entropy (as used in state-of-the-art methods; cf. \citep{nikitin2024kernel}) to predict the model's risk.

To further validate our theoretical findings, we compare the risk $\mathcal{R}_S$ and the calibration error $\operatorname{Cal}_S$, as the matrix scaling temperature is increased. The plots presented in Figure~\ref{fig:risk_calib_evece} (Appendix~\ref{app:additional_experiments}) confirm that both quantities are minimized at the same scaling temperature across all models and datasets, as predicted by Theorem~\ref{th:cal_risk_optim}.

Further, we are also interested in the implications of matrix calibration regarding downstream tasks.
One example is the usage of uncertainty scores for detecting answer correctness, which is evaluated via the AUROC \citep{kuhn2022semantic, gruber2024biasvariancecovariance, nikitin2024kernel, walha2025finegrained}.
Since there is no theoretical connection between calibration and AUROC, it is possible that calibration could have a negative impact on the downstream task. We compare the AUROC values in all settings before and after temperature scaling via risk minimization in Table~\ref{tbl:aurocs}.
As can be seen, even though we do not optimize temperature scaling towards the AUROC, improving matrix calibration has mostly a beneficial effect on detecting the model's answer correctness. Note that other state-of-the-art uncertainty measures for LLMs such as semantic entropy \citep{kuhn2022semantic} and kernel language entropy \citep{nikitin2024kernel} cannot be affected by matrix temperature scaling, as they are independent of the predicted density matrix and depend only on the generated answers directly, which is orthogonal to our approach. As shown in \citep{walha2025finegrained}, these approaches are outperformed by the entropy baseline. We further evaluate both approaches on TriviaQA in our setting and report the results in Table~\ref{tbl:kle_se_aurocs}. Indeed, the density-matrix-based entropy and maximum eigenvalue approaches outperform semantic entropy and kernel language entropy, further supporting the claim that these approaches achieve competitive, state-of-the-art performance. This motivates the need for a principled calibration method tailored to them. 

\begin{table}[h]
\caption{AUROC values ($\uparrow$) based on fuzzy answer correctness for Kernel Language Entropy (KLE) and Semantic Entropy (SE) on TriviaQA.}
\label{tbl:kle_se_aurocs}
\begin{center}
\begin{tabular}{lll}
\toprule
Model & KLE & SE \\
\midrule
Llama4 Maverick & \textbf{0.674} $\pm$ 0.001 & 0.673 $\pm$ 0.001 \\
Phi 4 & \textbf{0.807} $\pm$ 0.001 & 0.798 $\pm$ 0.001 \\
Phi 4 Mini & \textbf{0.797} $\pm$ 0.000 & 0.789 $\pm$ 0.001 \\
\bottomrule
\end{tabular}
\end{center}
\end{table}

Finally, we evaluate the effect of matrix temperature scaling on classical correctness calibration. The results are presented in Appendix~\ref{app:correctness_calibration}.

\begin{figure}[htbp]
    \centering
    \begin{subfigure}{0.235\textwidth}
        \centering
        \includegraphics[width=\linewidth]{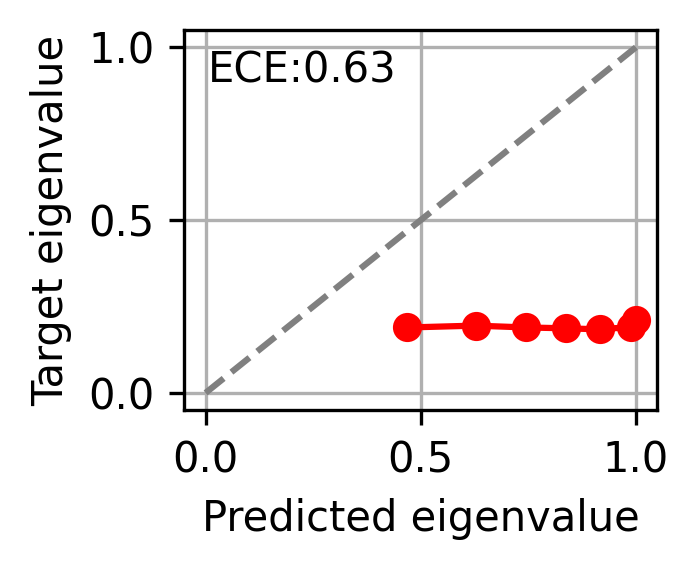}
        \caption{Before temperature scaling}
        \label{fig:naive_rel_diag_sub1}
    \end{subfigure}
    \begin{subfigure}{0.235\textwidth}
        \centering
        \includegraphics[width=\linewidth]{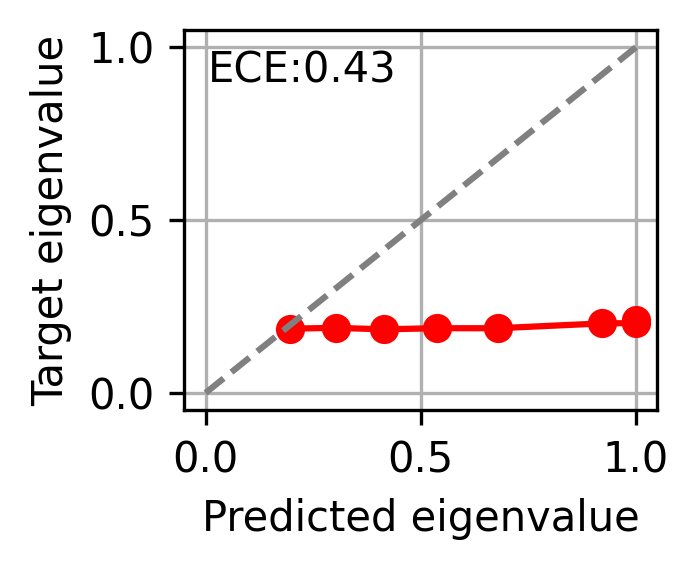}
        \caption{After temperature scaling}
        \label{fig:naive_rel_diag_sub_sub2}
    \end{subfigure}
    
    \caption{Reliability diagrams according to eigenvalue calibration as in Eq.~\ref{eq:EV_cal}.
    The ``plateau'' of eigenvalues along the x-axis is expected following Theorem~\ref{th:cal_inequality}.
    Comparing the distribution before and after temperature scaling indicates that the predicted eigenvalues are adjusted towards the diagonal, which reduces the respective ECE.}
    \label{fig:naive_rel_diagram}
\end{figure}

\begin{algorithm}
\SetAlgoLined
\KwIn{Target embeddings $Y \in \mathbb{R}^{n \times d}$ and predicted density matrices $D \in \mathbb{R}^{n \times d \times d}$ (data instances $n$, embedding dimension $d$), bin number $B$, cluster number $C$.}
\KwOut{predicted eigenvalues $\lambda_{\mathrm{preds}} \in \mathbb{R}^B$, target eigenvalues $\lambda_{\mathrm{targets}} \in \mathbb{R}^B$}

Compute max eigenvalues $E \gets \{ \lambda_{\max} ( D_i ) \mid i \in 1..n\}$\;
Compute equal mass bins $\mathbf{B}_1, \dots, \mathbf{B}_B \subset E$ \;
\For{$b \in 1..B$}{
    Select pred. density matrices $d_b$ w.r.t. $\mathbf{B}_b$ \;
    Compute corr. matrix $M \gets \operatorname{corr} \left( d_b \right)$ \;
    Compute hierarchical clusters $\mathbf{C}_1,\dots, \mathbf{C}_C \gets \operatorname{clusters(M)}$ \;
    \For{$c \in 1..C$}{
        Select target embeddings $Y_c$ w.r.t. $\mathbf{C}_c$ \;
        Compute target eigenvalue $\mathbf{T}_{b,c} \gets \lambda_{\max} \left( \frac{1}{m_{\mathrm{sel}}} Y_c Y_c^\intercal \right)$ \;
    }
    $\lambda_{\mathrm{preds},b} \gets \operatorname{avg} \left( \mathbf{B}_b \right)$ \;
    $\lambda_{\mathrm{targets},b} \gets \operatorname{avg} \left( \mathbf{T}_b \right)$ \;
}
\Return $\lambda_{\mathrm{preds}}, \lambda_{\mathrm{targets}}$\;
\caption{Reliability diagram for eigenvalues of density matrix predictions.}
\label{alg:rel_diag}
\end{algorithm}

\begin{figure*}[htbp]
    \centering
    \begin{subfigure}{0.81\textwidth}
        \centering
        \includegraphics[width=\linewidth]{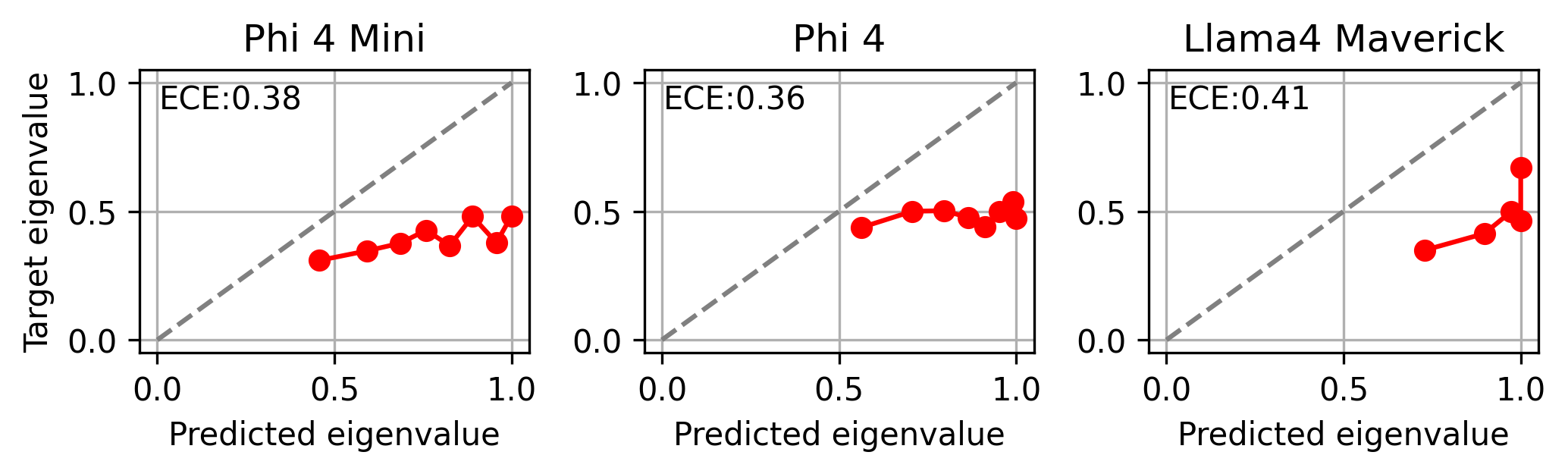}
        \caption{Before temperature scaling}
        \label{fig:rel_diag_nq}
    \end{subfigure} \\
    \begin{subfigure}{0.81\textwidth}
        \centering
        \includegraphics[width=\linewidth]{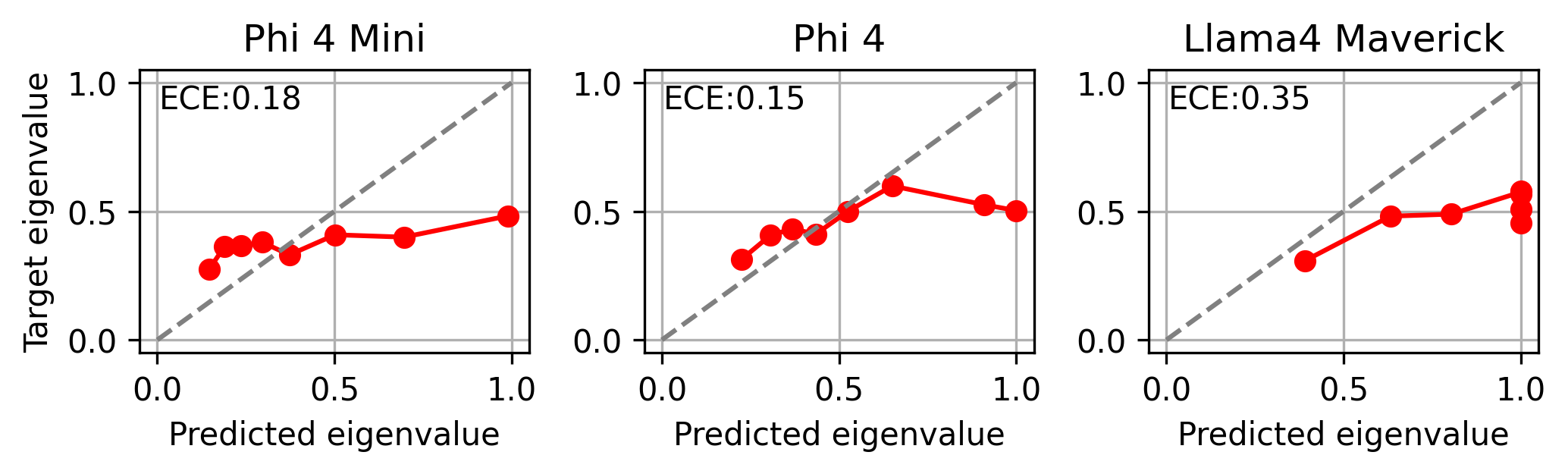}
        \caption{After temperature scaling}
        \label{fig:rel_diag_nq_TS}
    \end{subfigure}
    
    \caption{Reliability diagrams according to Algorithm~\ref{alg:rel_diag} of LLMs before and after temperature scaling.
    Figure~\ref{fig:rel_diag_nq} shows that all models are systematically overconfident in their predicted eigenvalue.
    This overconfidence is reduced via temperature scaling as seen in Figure~\ref{fig:rel_diag_nq_TS}.
    The respective ECE is also reduced in consequence.
    This improves the correctness of predicted eigenvalues by the models.
    }
    \label{fig:main_rel_diag}
\end{figure*}

\section{Reliability Diagrams for Eigenvalues}
\label{sec:rel_diagrams}

In classification, it is a common procedure to optimize calibration via risk minimization and to assess a model's calibration via reliability diagrams.
There, reliability diagrams are constructed by binning the probability predictions of the model and then computing the target frequency per bin.
In the following, we transfer this procedure to eigenvalues.

Theorem~\ref{th:cal_inequality} states that matrix calibration only translates to calibrated eigenvalues if we condition on the predicted density matrix.
Otherwise, we end up with an inequality.
The empirical implications of this theoretical result can be observed when plotting reliability diagrams for the maximum eigenvalue without information on the predicted density matrix.
Here, we bin according to the predicted maximum eigenvalue and then compute the target eigenvalue in each bin.
The resulting reliability diagram for Phi 4 Mini on TriviaQA is depicted in Figure~\ref{fig:naive_rel_diagram}.
The weighted gap between the diagonal line and the empirical target eigenvalues is included as the expected calibration error (ECE) value.
We identify that temperature scaling reduces some of the model overconfidence, which is also indicated via the ECE value.

Ideally, we want a reliability diagram for Equation~\ref{eq:EV_semi_cal}, i.e., we need to condition on information about the predicted confidence matrix beyond the maximum eigenvalue.
To achieve this, we propose Algorithm~\ref{alg:rel_diag}, which uses hierarchical clustering as an additional step during the binning procedure.
The clustering is done based on the predicted density matrices, which, in consequence, includes their information for computing the target eigenvalues.
We then compute the average target eigenvalue across clusters to ``collapse'' the information again on a single axis for plotting.
Therefore, instead of plotting $\lambda_{\max}\left( \mathbb{D}_{Y \mid \lambda_{\max} \left( d \left( X \right) \right)} \right)$ on the y-axis, we plot $\mathbb{E} \left[ \lambda_{\max}\left( \mathbb{D}_{Y \mid d \left( X \right)} \right) \mid \lambda_{\max} \left( d \left( X \right) \right) \right]$.
This is in line with Equation~\ref{eq:EV_semi_cal}.
The resulting reliability diagrams for $B=8$ bins and $C=5$ clusters are depicted in Figure~\ref{fig:main_rel_diag} (Natural Questions) and Figures~\ref{fig:entry_sub4} \& ~\ref{fig:entry_sub5} (Phi 4 Mini on TriviaQA).
Here, it is more pronounced how temperature scaling improves the calibration of the LLMs, both visually and in the ECE value.
We also provide the corresponding reliability diagrams for TriviaQA on all models in Appendix~\ref{app:additional_experiments}, which further confirm this observation. 

Further, we validate the choice of $B=8$ and $C=5$ for Algorithm~\ref{alg:rel_diag} using a sensitivity analysis. We also compare the performance of matrix temperature scaling to a sampling temperature calibration baseline for LLMs from \citet{lamb2025semantic}. Finally, we reproduce the main results of this paper using a different embedding model, and achieve similar outcomes. All the details and results of these experiments are included in Appendix~\ref{app:additional_experiments}. 

\section{Conclusion}

We introduced a principled framework for calibrating eigenvalues of semantic embeddings, bridging a key gap between conventional calibration and state-of-the-art uncertainty quantification methods for LLMs.
Our results establish entropy–risk equivalence under calibration, prove a central calibration inequality for eigenvalues, and show that temperature scaling reduces calibration error.
Evaluations confirmed that modern LLMs are systematically overconfident and that recalibration allows for interpreting average entropy as risk, besides improving several other metrics.
Therefore, eigenvalue calibration opens new directions for embedding-based uncertainty quantification towards more reliable AI systems.

\begin{acknowledgements}
Co-funded by the European Union (ERC, TAIPO, 101088594 to FB). Views and opinions expressed are however those of the authors only and do not necessarily reflect those of the European Union or the European Research Council. Neither the European Union nor the granting authority can be held responsible for them.
\end{acknowledgements}

\bibliography{bibliography}

\newpage

\onecolumn

\title{Appendix}
\maketitle

\appendix
In this Appendix, we offer experimental details in Appendix~\ref{app:exp_details}, additional results in Appendix~\ref{app:additional_experiments}, and missing proofs from the main paper in Appendix~\ref{app:proofs}.

\section{EXTENDED EXPERIMENTAL DETAILS}
\label{app:exp_details}
In this section, we discuss more details on the experimental setup.
The source code for all experiments is openly available at \url{https://github.com/MLO-lab/matrix_eigenvalue_calibration}.

For each development subset (of size $n=300$) and each model, we use $m=100$ sampled answers.
For each test subset (of size $n=1700$) and each model, we use $m=20$ sampled answers.

For computing the log score, we add a small $\epsilon$=1e-10 to all eigenvalues equal to zero, to avoid computing the logarithm of zero.

To compute AUROC values in Table \ref{tbl:aurocs}, we require a binary indicator of model correctness for each question.
This is obtained by generating a single “standard” answer at temperature 0.1 and comparing it against the ground-truth reference using a fuzzy matching criterion based on the ROUGE-L score \cite{kuhn2022semantic}.

The error bars for the risks and entropy values in all figures are computed via their standard deviation.
The error bars for the AUROC values are computed via bootstrap sampling of $B=20$ subsets.

\section{ADDITIONAL EXPERIMENTS AND RESULTS}
In this section, we present additional experimental results.
\label{app:additional_experiments}

\subsection{EMPIRICAL ESTIMATOR CONVERGENCE}
\begin{figure*}[htbp]
    \centering
    \includegraphics[width=0.4\linewidth]{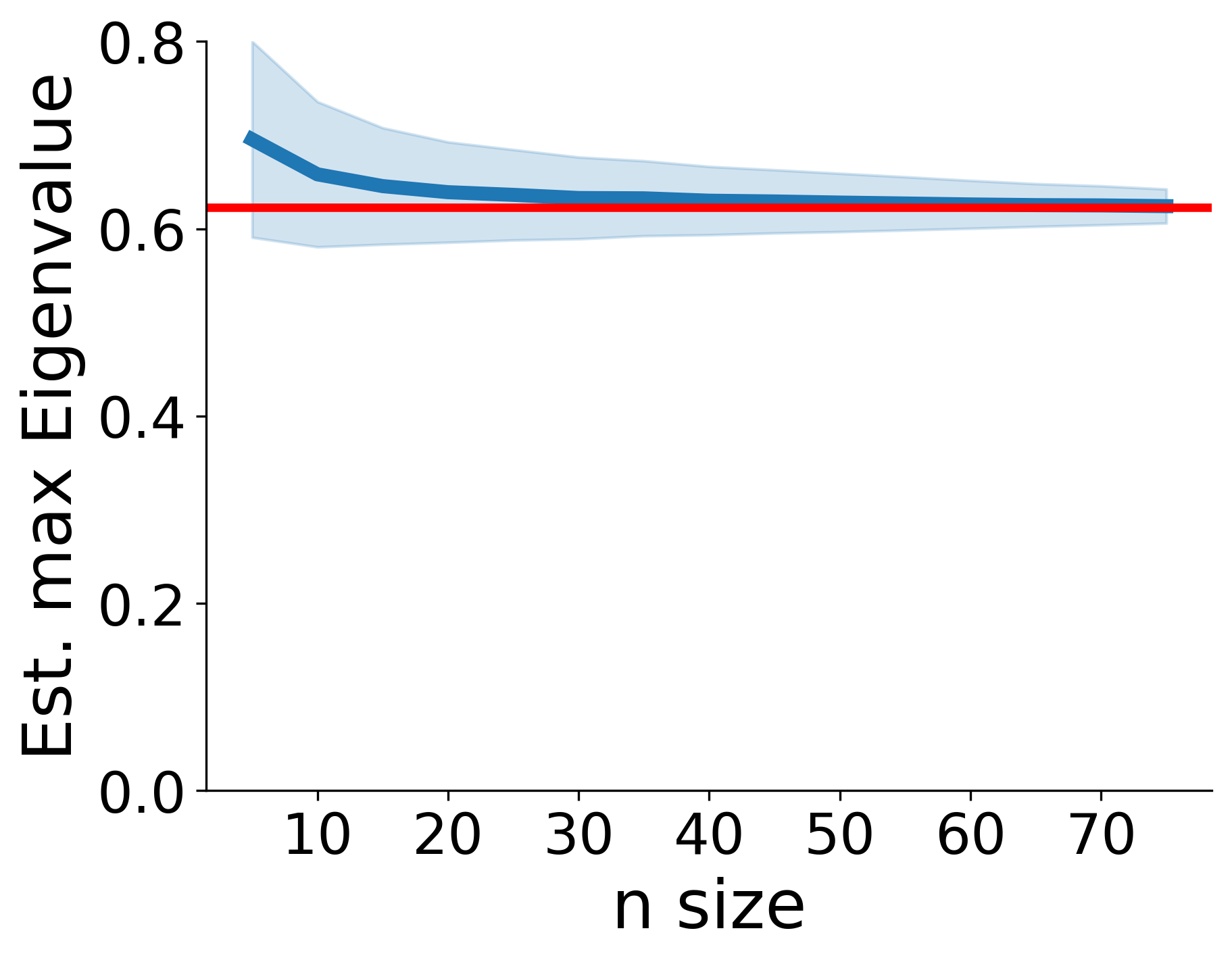}
    \caption{Estimating the maximum eigenvalue has a positive bias but converges quickly to the ground truth (red line) for $m\geq20$ data instances.}
    \label{fig:max_EV_est_sim}
\end{figure*}

We assess the empirical estimator convergence by sampling a subset of answers repeatedly from a set of 300 sampled answers from an LLM.
In Figure~\ref{fig:max_EV_est_sim} are shown the results, where we include the standard deviation of the estimator.
As can be seen, the estimator has a positive bias but converges quickly to the ground truth.
Based on these results, we picked a sampling size of $m=20$ in our main experiments.

\subsection{RELIABILITY RESULTS FOR TRIVIAQA}
\begin{figure*}[htbp]
    \centering
    \begin{subfigure}{0.9\textwidth}
        \centering
        \includegraphics[width=\linewidth]{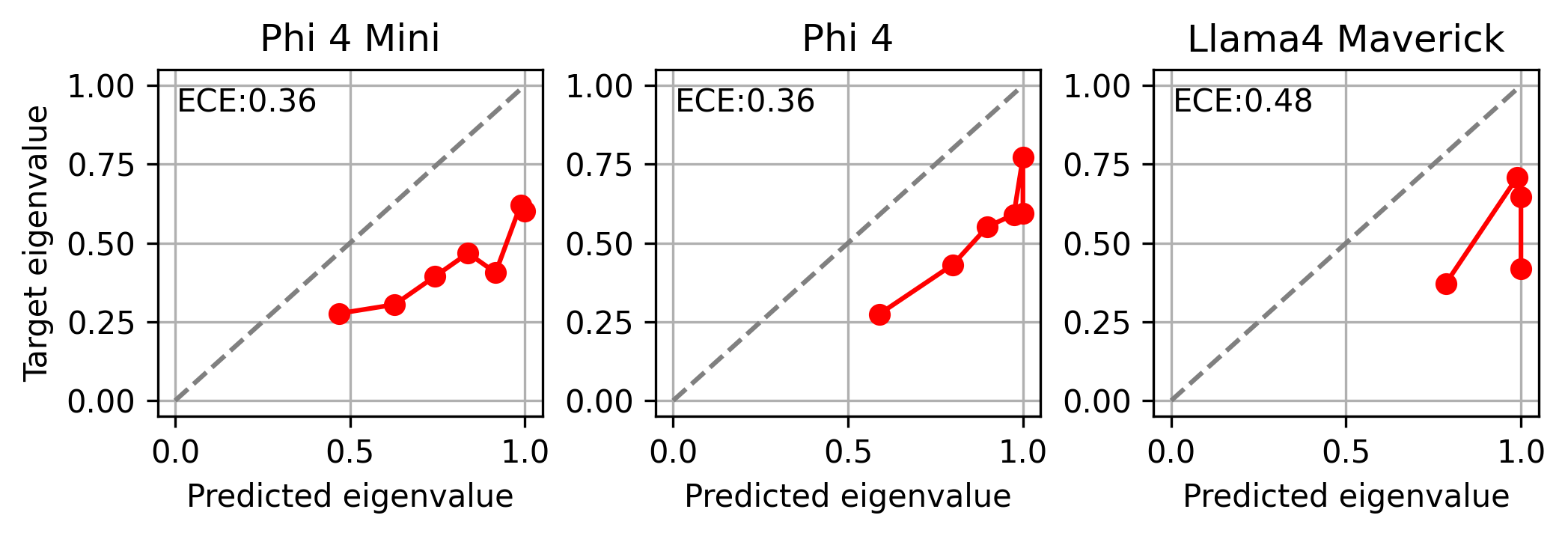}
        \caption{Before temperature scaling}
        \label{fig:rel_diag_trivia}
    \end{subfigure} \\
    \begin{subfigure}{0.9\textwidth}
        \centering
        \includegraphics[width=\linewidth]{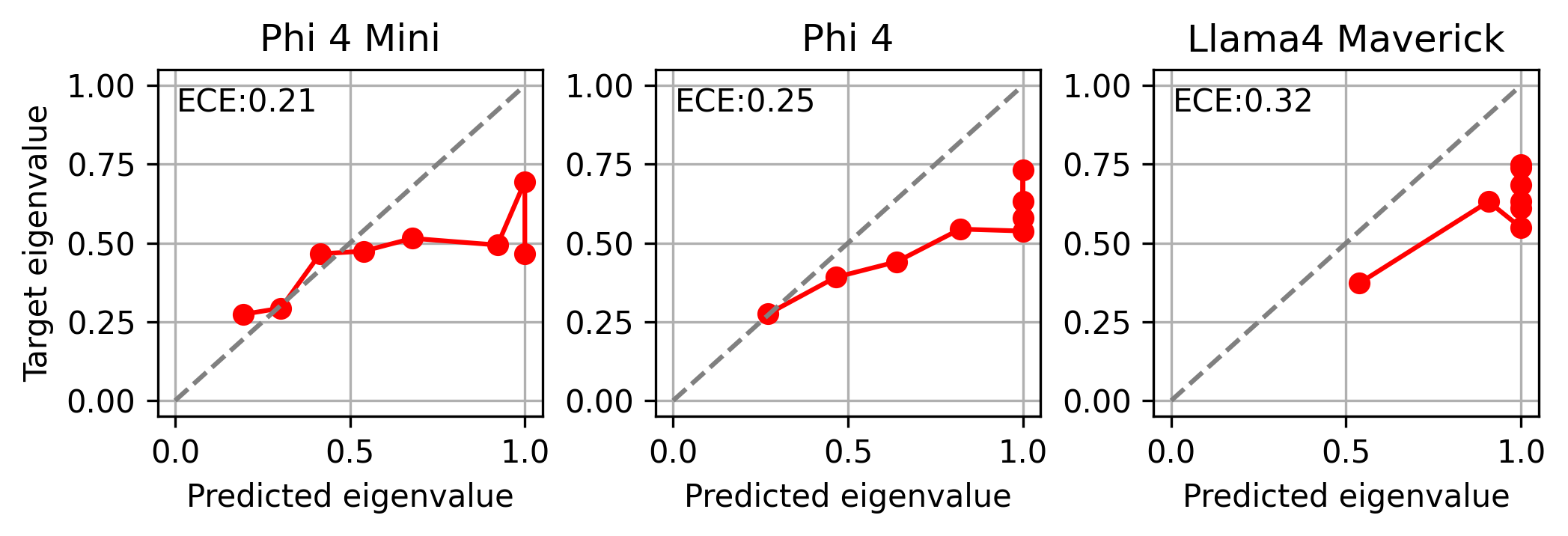}
        \caption{After temperature scaling}
        \label{fig:rel_diag_trivia_TS}
    \end{subfigure}
    
    \caption{Reliability diagrams for TriviaQA according to Algorithm~\ref{alg:rel_diag} of LLMs before and after temperature scaling.
    Figure~\ref{fig:rel_diag_trivia} shows that all models are systematically overconfident in their predicted eigenvalue.
    This overconfidence is reduced via temperature scaling as seen in Figure~\ref{fig:rel_diag_trivia_TS}.
    The respective ECE is also reduced in consequence.
    This improves the correctness of predicted eigenvalues by the models.}
    \label{fig:app_rel_diag}
\end{figure*}

We compare the change of eigenvalue calibration based on temperature scaling via risk minimization for TriviaQA in the same manner as Natural Questions in Section~\ref{sec:optim_calibration}.
The results are depicted in Figure~\ref{fig:app_rel_diag}.
As can be seen, the results are similar to the ones discussed in the main paper and support our developed theory and methodology.
Specifically, temperature scaling adjusts the calibration of the LLMs, which are systematically overconfident.

\subsection{CALIBRATION TOWARDS CORRECTNESS}
\label{app:correctness_calibration}
\begin{figure*}[htbp]
    \centering
    \begin{subfigure}{\textwidth}
        \centering
        \includegraphics[width=\linewidth]{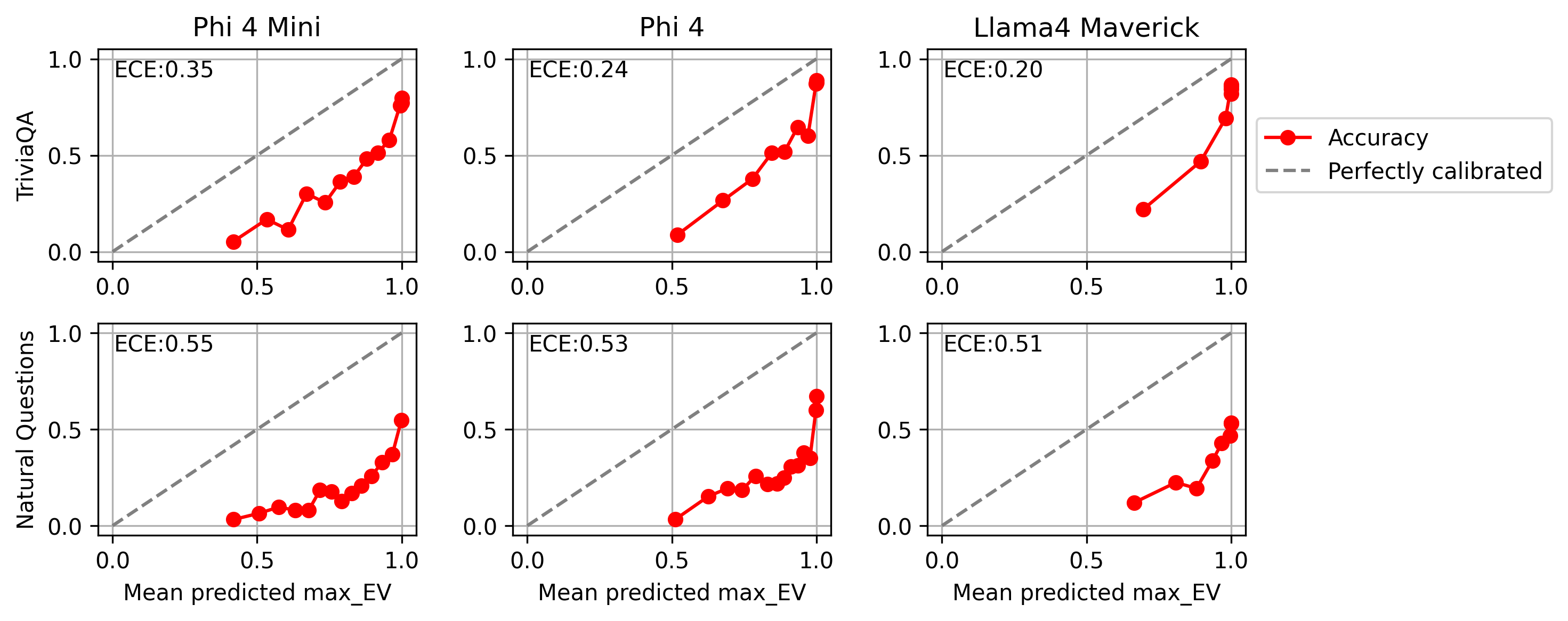}
        \caption{Before temperature scaling}
    \label{fig:app1_sub1}
    \end{subfigure} \\
    \begin{subfigure}{\textwidth}
        \centering
        \includegraphics[width=\linewidth]{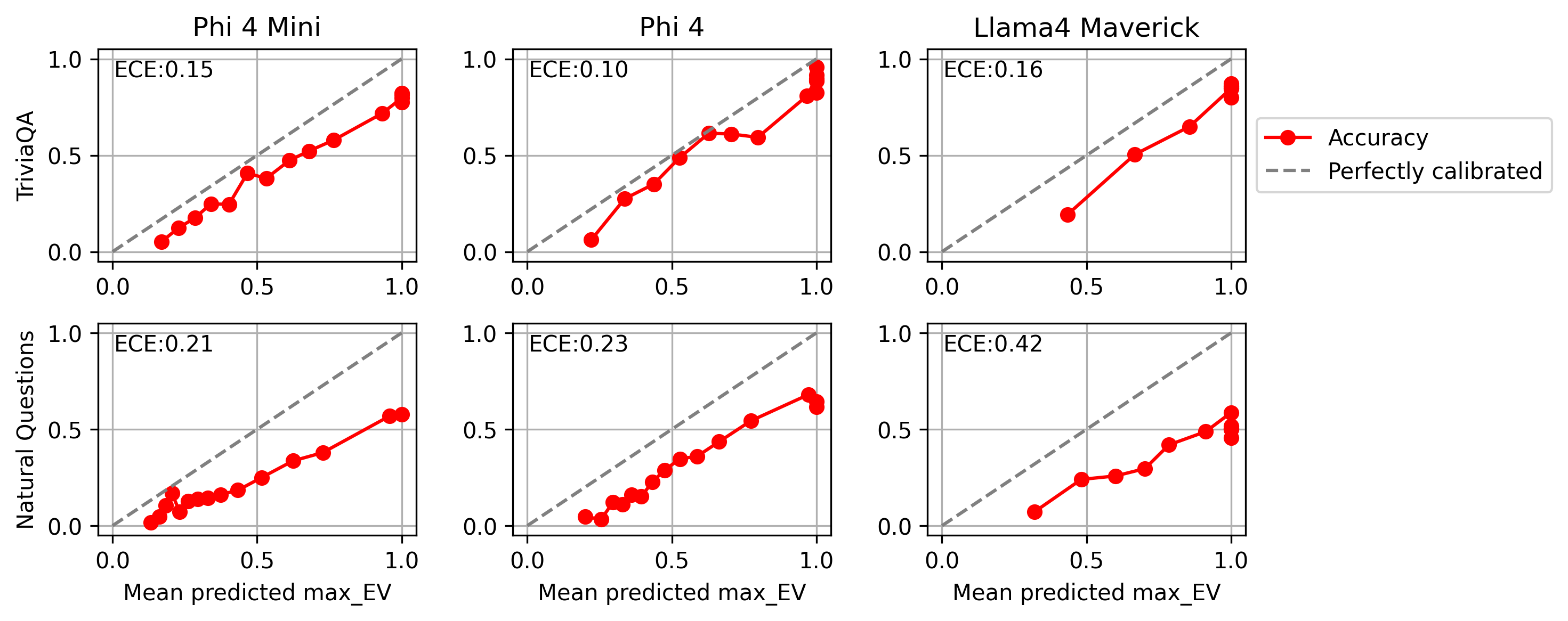}
        \caption{After temperature scaling}
        \label{fig:app1_sub2}
    \end{subfigure}
    
    \caption{Reliability diagrams for conventional answer correctness.
    Matrix temperature scaling also has beneficial effects on how representative predicted eigenvalues are for the answer correctness, even though we do not directly optimize towards this objective.}
    \label{fig:app1}
\end{figure*}
As an additional downstream task analysis, we also evaluate how calibrating eigenvalues influences the predictiveness of the eigenvalues regarding the likelihood of answer correctness.
For this, we construct conventional reliability diagrams with the predicted eigenvalue on the x-axis and the average correctness given only this prediction on the y-axis.
The results are seen in Figure~\ref{fig:app1}.
Note that we still optimize temperature scaling via the risk of a matrix score, i.e., we do not change the optimization objective from the main paper.
As we can see, calibrating via risk minimization also improves the calibration with the answer correctness likelihood, even though we do not directly optimize towards this objective.

Based on these results and the AUROC values in Table~\ref{tbl:aurocs}, we conclude that optimizing matrix calibration via risk minimization does not deteriorate any downstream tasks, and, in most cases, even improves on them.

\subsection{RISK, CALIBRATION ERROR, AND EIGENVALUE ECE ARE CO-MINIMIZED}
\begin{figure*}[!htbp]
    \centering
    \includegraphics[width=\linewidth]{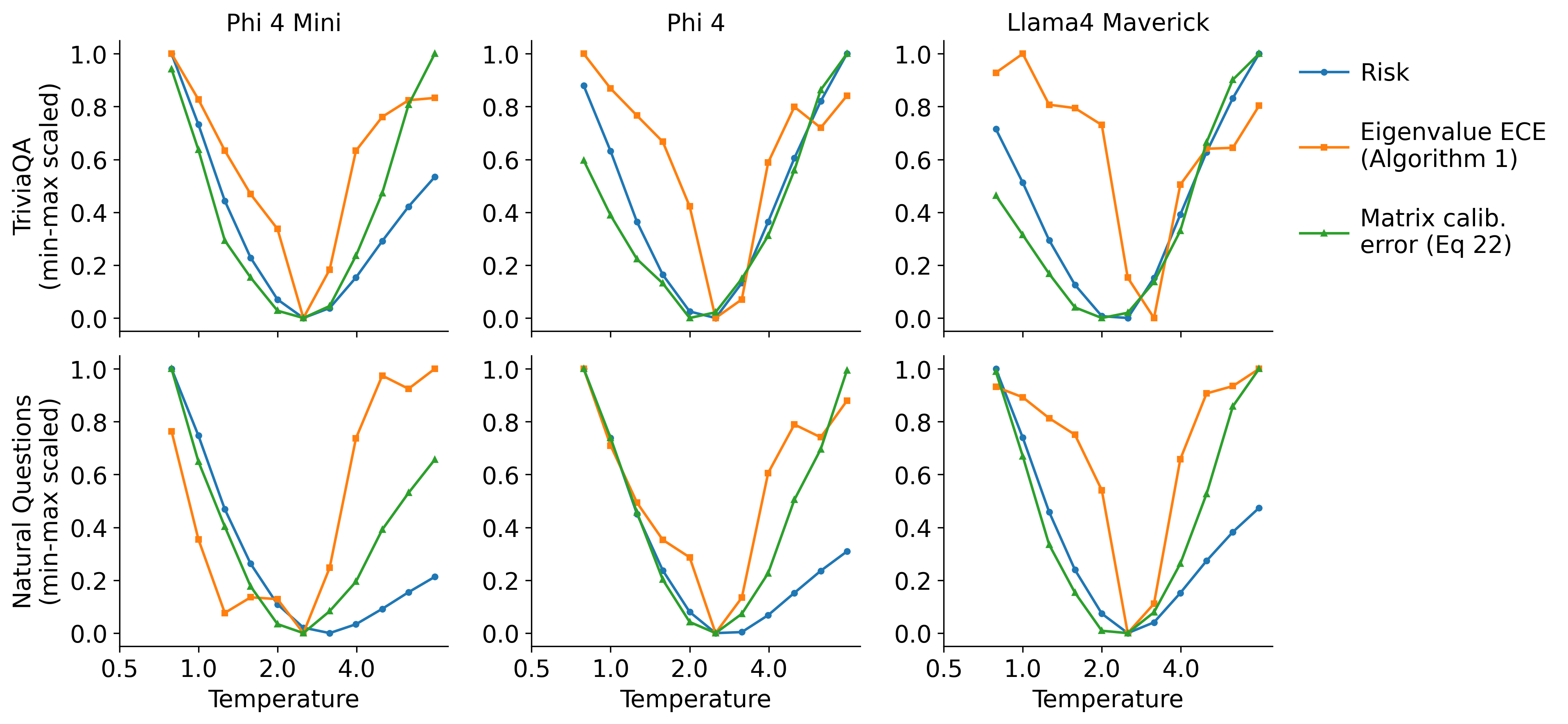}
    
    \caption{Scaled values of the risk $\mathcal{R}_S$, matrix calibration error $\operatorname{Cal}_S$, and the eigenvalue ECE (as defined in Algorithm~\ref{alg:rel_diag}). As predicted by Theorem~\ref{th:cal_risk_optim}, all quantities are co-minimized at essentially the same temperature (with small deviations due to estimation noise).}
    \label{fig:risk_calib_evece}
\end{figure*}

By Theorem~\ref{th:cal_risk_optim}, minimizing the risk $\mathcal{R}_S$ via matrix temperature scaling is equivalent to minimizing the matrix calibration error $\operatorname{Cal}_S$. Theorem~\ref{th:cal_inequality} further suggests that this equivalence should also manifest as a minimal eigenvalue ECE, as defined in Algorithm~\ref{alg:rel_diag}.

We apply matrix temperature scaling across all models and datasets in our evaluation, tracking each of these quantities throughout. The results, shown in Figure~\ref{fig:risk_calib_evece}, confirm the theoretical prediction. To estimate the matrix calibration error $\operatorname{Cal}_S$, we cluster the predicted density matrices $d(X)$ across the full dataset, which is required to obtain a reliable estimate of $\mathbb{D}_{Y \mid d(X)}$. Estimating the eigenvalue ECE via Algorithm~\ref{alg:rel_diag}, by contrast, requires a bin-then-cluster procedure, in which clustering of $d(X)$ is restricted to within each bin. Consequently, the eigenvalue ECE estimate is inherently noisier than the matrix calibration error estimate.

Finally, this experiment complements Figure~\ref{fig:TS_risk_gt_dummy} by showing that the matrix generalization of the cross-entropy loss recovers the optimal temperature not only for constant density matrix predictors, but also for realistic LLM-based density matrix predictors.

\subsection{ADDITIONAL ANALYSIS FOR ALGORITHM~\ref{alg:rel_diag}}

\begin{table}[h]
\caption{Eigenvalue ECE (as computed by Algorithm~\ref{alg:rel_diag}) across different values of $B$ (number of bins) and $C$ (number of clusters) on the Natural Questions dataset using Phi 4 Mini. ECE is stable across a wide range of $(B, C)$ once $C \geq 5$.}
\label{tbl:bce_ece}
\begin{center}
\begin{tabular}{lccccccccc}
\toprule
$B$ & \multicolumn{3}{c}{6} & \multicolumn{3}{c}{8} & \multicolumn{3}{c}{10} \\
\cmidrule(lr){2-4} \cmidrule(lr){5-7} \cmidrule(lr){8-10}
$C$ & 3 & 5 & 7 & 3 & 5 & 7 & 3 & 5 & 7 \\
\midrule
\textbf{Eigenvalue ECE} & 0.293 & 0.198 & 0.170 & 0.277 & 0.188 & 0.220 & 0.236 & 0.222 & 0.211 \\
\bottomrule
\end{tabular}
\end{center}
\end{table}

We propose to evaluate the sensitivity of Algorithm~\ref{alg:rel_diag} to the choice of the binning and clustering parameters $B$ and $C$. To this end, we run the algorithm for different values of these parameters and evaluate the eigenvalue ECE. The results are shown in Table~\ref{tbl:bce_ece}. The low number of clusters $C=3$ is the only failure mode in our analysis, where over-merging makes the per-cluster target eigenvalues near-uniform (consistent with Figure~\ref{fig:naive_rel_diagram}). For $C \geq 5$, this problem disappears and we observe stable eigenvalue ECE values across different configurations. Our default values $B=8$, $C=5$ are chosen on that basis and do not require tuning per dataset or model.

\begin{table}[h]
\caption{Within-cluster and within-bin pairwise similarity of predicted density matrices $d(X)$ across all models and datasets. Matrices grouped in one cluster are substantially more aligned with each other than matrices in one bin, highlighting the effectiveness of our clustering to approximate the conditioning on $d(X)$.}
\label{tbl:pairwise_similarity}
\begin{center}
\begin{tabular}{llll}
\toprule
 &  & Within cluster pairwise similarity & Within bin pairwise similarity \\
Dataset & Model &  &  \\
\midrule
\multirow[t]{3}{*}{Natural Questions} & Llama4 Maverick & 0.490 & 0.036 \\
 & Phi 4 & 0.351 & 0.033 \\
 & Phi 4 Mini & 0.362 & 0.060 \\
\cline{1-4}
\multirow[t]{3}{*}{TriviaQA} & Llama4 Maverick & 0.401 & 0.046 \\
 & Phi 4 & 0.421 & 0.046 \\
 & Phi 4 Mini & 0.318 & 0.064 \\
\cline{1-4}
\bottomrule
\end{tabular}
\end{center}
\end{table}

We also perform a clustering-coherence check showing that density matrices grouped together by Algorithm~\ref{alg:rel_diag} are substantially more aligned with each other than matrices that merely share a confidence bin, supporting the claim that the hierarchical clustering we use is a practical approximation to conditioning on $d(X)$. The results are presented in Table~\ref{tbl:pairwise_similarity}.

\subsection{SAMPLING TEMPERATURE CALIBRATION BASELINE}

\begin{figure*}[!htbp]
    \centering
    \includegraphics[width=0.33\linewidth]{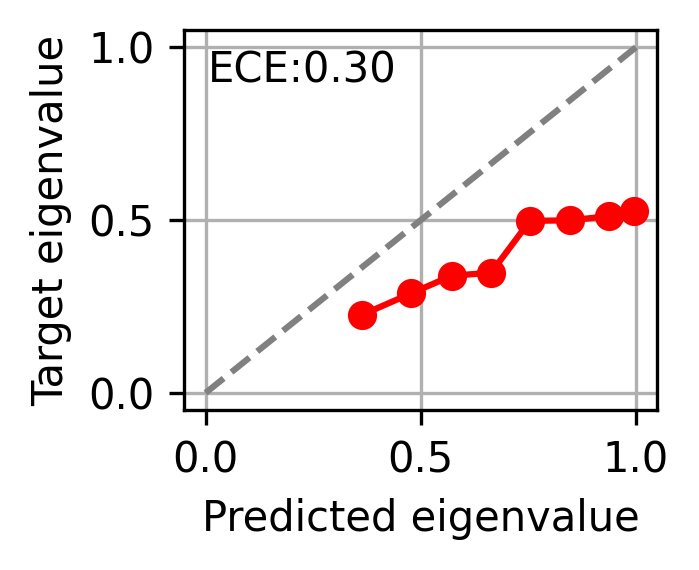}
    
    \caption{Eigenvalue-based reliability diagram (as computed by Algorithm~\ref{alg:rel_diag}) after applying the sampling temperature calibration baseline for the Figure~\ref{fig:entry_sub4} and Figure~\ref{fig:entry_sub5} setup (TriviaQA using Phi 4 Mini). Calibrating sampling temperature does reduce the eigenvalue ECE but the performance is limited compared to our matrix temperature scaling approach.}
    \label{fig:samp_temp_baseline}
\end{figure*}

In this paper, we provide a novel notion of matrix calibration, and therefore there are no existing baselines that are designed for matrix calibration. Nevertheless, there are already existing works in the LLM calibration literature that consider other notions of calibration. For example, \citet{lamb2025semantic} compute the optimal sampling temperature that minimizes the calibration error against correctness, using the probability of the largest cluster from semantic entropy \citep{kuhn2022semantic} as a confidence measure. We evaluate how this approach affects matrix calibration and compare it to matrix temperature scaling. First, similarly to \citet{lamb2025semantic}, we identify the optimal sampling temperature $t^*$ by minimizing the expected calibration error against correctness. Then, we use the temperature $t^*$ to generate answers and provide predicted density matrices $d(X)$, for which we can evaluate matrix calibration using Algorithm~\ref{alg:rel_diag}. The resulting reliability diagram is depicted in Figure~\ref{fig:samp_temp_baseline} and corresponds to the TriviaQA and Phi 4 Mini setup. Compared to the uncalibrated setup (Figure~\ref{fig:entry_sub4}) sampling temperature calibration \citep{lamb2025semantic} reduced eigenvalue ECE from 0.36 to 0.30. However, as expected, matrix temperature scaling (Figure~\ref{fig:entry_sub5}) outperforms this baseline and reduces eigenvalue ECE even further to 0.21.

\subsection{ADDITIONAL EMBEDDING MODEL}
\begin{figure*}[htbp]
    \centering
    \begin{subfigure}{\textwidth}
        \centering
        \includegraphics[width=\linewidth]{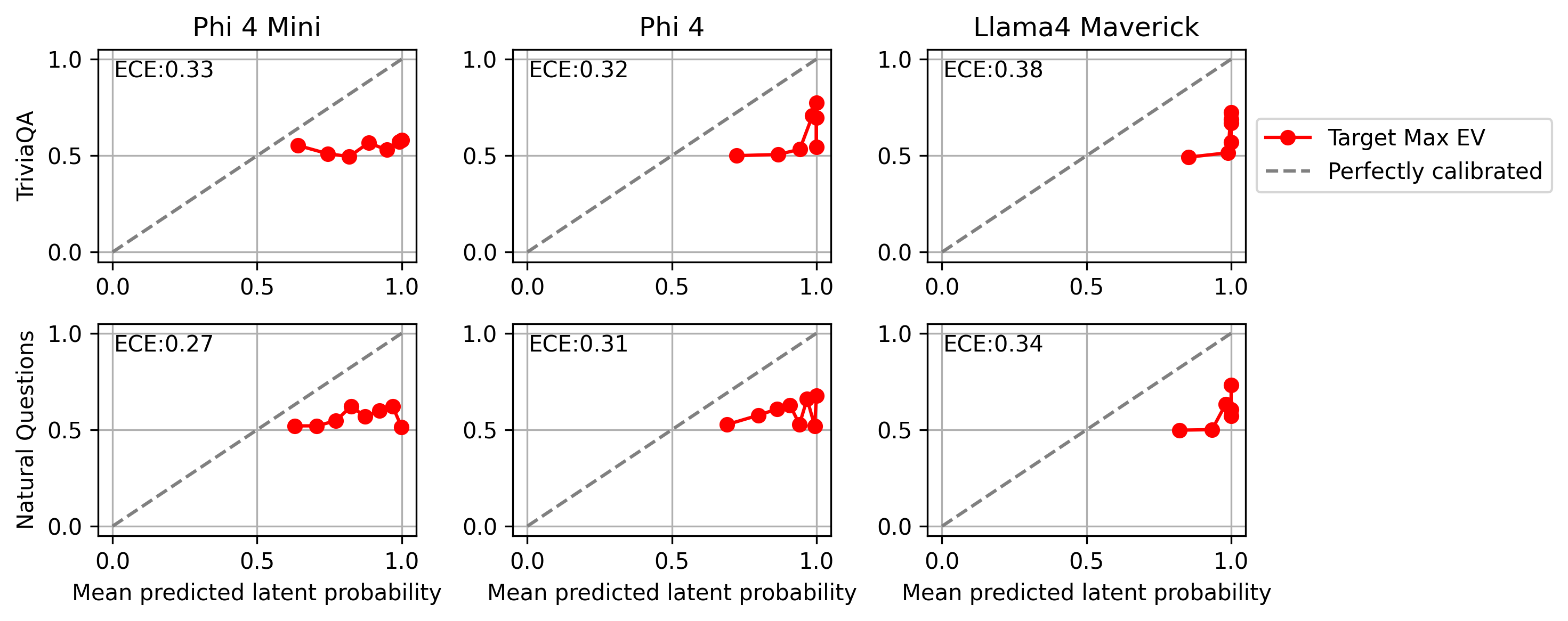}
        \caption{Before temperature scaling}
    \label{fig:jina_sub1}
    \end{subfigure} \\
    \begin{subfigure}{\textwidth}
        \centering
        \includegraphics[width=\linewidth]{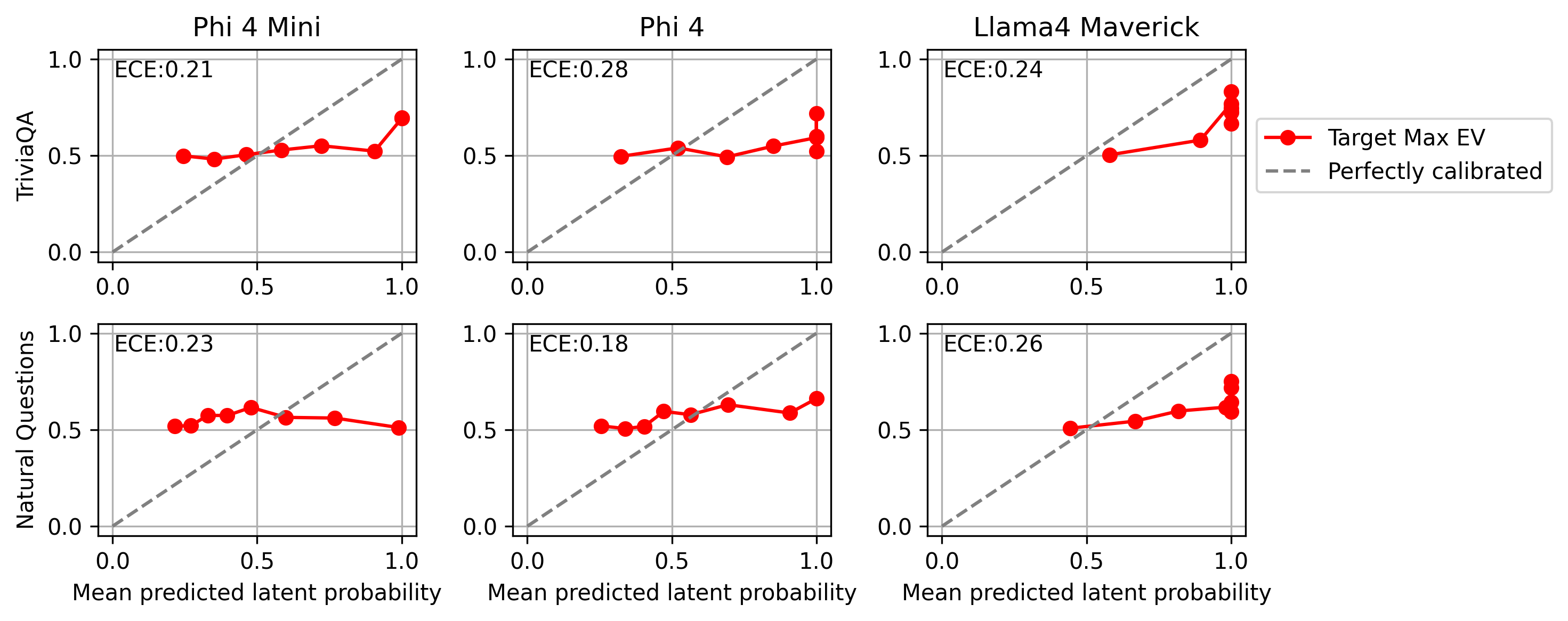}
        \caption{After temperature scaling}
        \label{fig:jina_sub2}
    \end{subfigure}
    
    \caption{Reliability diagrams according to Algorithm~\ref{alg:rel_diag} of LLMs before and after temperature scaling using \texttt{jinaai/jina-embeddings-v5-text-nano} embeddings. Similarly to the main results, LLMs are always overconfident, and matrix temperature scaling effectively reduces this overconfidence, achieving lower eigenvalue ECEs.}
    \label{fig:jina}
\end{figure*}

\begin{table*}[htbp]
\caption{AUROC values ($\uparrow$) based on fuzzy answer correctness with maximum eigenvalue and entropy as uncertainty scores before and after temperature scaling, using the \texttt{jinaai/jina-embeddings-v5-text-nano} embedding model. Using the optimal temperature based on risk minimization improves the AUROC in many cases.}
\label{tbl:aurocs_jina}
\begin{center}
\begin{tabular}{llllll}
\toprule
 &  & Eigenvalue & Eigenvalue TS & Entropy & Entropy TS \\
Dataset & Model &  &  &  &  \\
\midrule
\multirow[t]{3}{*}{Natural Questions} & Llama4 Maverick & \textbf{0.658} $\pm$ 0.0 & 0.646 $\pm$ 0.0 & \textbf{0.646} $\pm$ 0.0 & 0.642 $\pm$ 0.0 \\
 & Phi 4 & 0.744 $\pm$ 0.001 & \textbf{0.795} $\pm$ 0.0 & 0.763 $\pm$ 0.001 & \textbf{0.797} $\pm$ 0.0 \\
 & Phi 4 Mini & \textbf{0.765} $\pm$ 0.001 & 0.764 $\pm$ 0.001 & \textbf{0.768} $\pm$ 0.001 & 0.748 $\pm$ 0.001 \\
\cline{1-6}
\multirow[t]{3}{*}{TriviaQA} & Llama4 Maverick & 0.686 $\pm$ 0.001 & \textbf{0.699} $\pm$ 0.001 & 0.696 $\pm$ 0.001 & \textbf{0.697} $\pm$ 0.001 \\
 & Phi 4 & 0.816 $\pm$ 0.001 & \textbf{0.825} $\pm$ 0.001 & \textbf{0.825} $\pm$ 0.001 & 0.818 $\pm$ 0.001 \\
 & Phi 4 Mini & 0.798 $\pm$ 0.001 & \textbf{0.799} $\pm$ 0.001 & \textbf{0.801} $\pm$ 0.0 & 0.79 $\pm$ 0.001 \\
\cline{1-6}
\bottomrule
\end{tabular}
\end{center}
\end{table*}

We reproduce the main experiments of the paper using another embedding model \texttt{jinaai/jina-embeddings-v5-text-nano} from \citet{akram2026jina}. The reliability diagrams before and after matrix temperature scaling are depicted in Figure~\ref{fig:jina}, and the AUROC values are provided in Table~\ref{tbl:aurocs_jina}. Overall, we can reach the same conclusion: LLMs are always overconfident before calibration, and matrix temperature scaling helps reduce this overconfidence considerably, while also improving AUROCs of spectral uncertainty quantification methods in many cases. These results are embedder-agnostic and they are driven by the LLM's answer distribution rather than by the particular embedding model we use.
\section{MISSING PROOFS}
\label{app:proofs}

In this section, we present the missing proofs from the main paper.

\subsection{Proof of Theorem~\ref{th:cal_inequality}}

\begin{theorem}[Theorem~\ref{th:cal_inequality} restated]
    Let $f$ be an LLM and $d$ its respective density matrix predictor.
    If $f$ is canonically calibrated then $d$ is matrix calibrated.
    Further, if $d$ is matrix calibrated, then
    \begin{equation}
        \lambda_{\max} \left(  \mathbb{D}_{Y \mid d \left( X \right)} \right) \overset{a.s.}{=} \Lambda_X,
    \end{equation}
    with $\Lambda_X \coloneqq \lambda_{\max} \left( d \left( X \right) \right)$, but also
    \begin{equation}
        \lambda_{\max} \left(  \mathbb{D}_{Y \mid \Lambda_X} \right) \overset{a.s.}{\leq} \Lambda_X.
    \end{equation}
\end{theorem}

As a further minor assumption, which we omitted for brevity, we require that all integrals in the following are finite.

\begin{proof}
We write $P \coloneqq f \left( X \right)$ for brevity.
We write $c = \int_{S_d} yy^\intercal \mathrm{d} p \left( y \right)$ for any $p \in \mathcal{P}$ (which we assume exists in $\mathbb{R}^{d \times d}$).
For every $p \in \operatorname{supp} \mathbb{P}_P$ holds
\begin{equation}
\begin{split}
    & \mathbb{P} \left( Y \mid P = p \right) = p \\
    \overset{(i)}{\implies} & \mathbb{E}_{P} \left[ \mathbb{P} \left( Y \mid P \right) \mid d \left( X \right) = c \right] = \mathbb{E} \left[ P \mid d \left( X \right) = c \right] \\
    \implies & \mathbb{P} \left( Y \mid d \left( X \right) = c \right) = \mathbb{E} \left[ P \mid d \left( X \right) = c \right] \\
    \overset{(i)}{\implies} & \int_{S_d} yy^\intercal \mathrm{d}
    \mathbb{P}_{Y \mid d \left( X \right) = c} \left( y \right) = \int_{S_d} yy^\intercal \mathrm{d}
    \mathbb{E} \left[ P \mid d \left( X \right) = c \right] \left( y \right) \\
    \iff & \mathbb{D}_{Y \mid d \left( X \right) = c} = \mathbb{E} \left[ \int_{S_d} yy^\intercal \mathrm{d}
    P \left( y \right) \mid d \left( X \right) = c \right] \\
    \iff & \mathbb{D}_{Y \mid d \left( X \right) = c} = \mathbb{E} \left[ d \left( X \right) \mid d \left( X \right) = c \right] \\
    \iff & \mathbb{D}_{Y \mid d \left( X \right) = c} = c \\
    & \implies \lambda_{\max} \left( \mathbb{D}_{Y \mid d \left( X \right) = c} \right) = \lambda_{\max} \left( c \right).
\end{split}
\end{equation}
Since $p \in \operatorname{supp} \mathbb{P}_P$ is arbitrary, it follows from the last line that
$\lambda_{\max} \left(  \mathbb{D}_{Y \mid d \left( X \right)} \right) \overset{a.s.}{=} \Lambda_X$.

To show the inequality, note that $\lambda_{\max}$ is a convex function, which leads to
\begin{equation}
\begin{split}
    \mathbb{E} \left[ \lambda_{\max} \left(  \mathbb{D}_{Y \mid d \left( X \right)} \right) \mid \Lambda_X \right] \geq \lambda_{\max} \left( \mathbb{E} \left[  \mathbb{D}_{Y \mid d \left( X \right)} \mid \Lambda_X \right] \right) = \lambda_{\max} \left( \mathbb{D}_{Y \mid \Lambda_X} \right).
\end{split}
\end{equation}
Therefore, $\lambda_{\max} \left(  \mathbb{D}_{Y \mid \Lambda_X} \right) \overset{a.s.}{\leq} \Lambda_X$.
This is essentially a special case of \citep[Theorem 1]{gruber2024novel}.
\end{proof}

\subsection{Proof of Proposition~\ref{prop:mats=ps}}

\begin{proposition}[Proposition~\ref{prop:mats=ps} restated]
    A proper matrix score $\mathbf{S} \colon \mathbb{H}_d^\Delta \to \mathbb{R}^{d\times d}$ generates a proper score $S \colon \mathcal{P}_d \times \mathcal{Y}_d \to \mathbb{R}$ with respect to the set of distributions $\mathcal{P}_d \coloneqq \left\{ P \mid \int_{S_d} \left\lVert x \right\rVert_2^2 \mathrm{d} P \left( x \right) < \infty\right\}$ defined on the hypersphere and $\mathcal{Y}_d = \mathbb{R}^d$.
    The respective proper score is given by
    \begin{equation}
        S \left( P, y \right) \coloneqq y^\intercal \mathbf{S} \left( \mathbb{E}_{x \sim P} \left[ xx^\intercal \right] \right) y.
    \end{equation}
\end{proposition}

\begin{proof}
    First, note that $S$ maps $\mathcal{P}_d \times \mathcal{Y}_d \to \mathbb{R}$.
    Now, we only need to show that for any $P,Q \in \mathcal{P}$ it holds $\mathbb{E}_{Y \sim Q} \left[ S \left( P, Y \right) \right] \geq \mathbb{E}_{Y \sim Q} \left[ S \left( Q, Y \right) \right]$.
    We will write $\mathbf{D}_Q \coloneqq \mathbb{E}_{x \sim Q} \left[ xx^\intercal \right]$ for any $Q \in \mathcal{P}_d$.
    It holds
    \begin{equation}
    \begin{split}
        \mathbb{E}_{\mathbf{Y} \sim Q} \left[ S \left( Q, \mathbf{Y} \right) \right] & = \mathbb{E}_{\mathbf{Y} \sim Q} \left[ \mathbf{Y}^\intercal \mathbf{S} \left( \mathbb{E}_{x \sim Q} \left[ xx^\intercal \right] \right) \mathbf{Y} \right] \\
        & = \operatorname{tr} \left( \mathbf{S} \left( \mathbf{D}_Q \right) \mathbb{E}_{\mathbf{Y} \sim Q} \left[ \mathbf{Y}\mathbf{Y}^\intercal \right] \right) \\
        & = \operatorname{tr} \left( \mathbf{S} \left( \mathbf{D}_Q \right) \mathbf{D}_Q \right) \\
        & \leq \operatorname{tr} \left( \mathbf{S} \left( \mathbf{D}_P \right) \mathbf{D}_Q \right) \\
        & = \operatorname{tr} \left( \mathbf{S} \left( \mathbb{E}_{x \sim P} \left[ xx^\intercal \right] \right) \mathbb{E}_{\mathbf{Y} \sim Q} \left[ \mathbf{Y}\mathbf{Y}^\intercal \right] \right) \\
        & = \mathbb{E}_{\mathbf{Y} \sim Q} \left[ \mathbf{Y}^\intercal \mathbf{S} \left( \mathbb{E}_{x \sim P} \left[ xx^\intercal \right] \right) \mathbf{Y} \right] \\
        & = \mathbb{E}_{\mathbf{Y} \sim Q} \left[ S \left( P, \mathbf{Y} \right) \right],
    \end{split}
    \end{equation}
    where the inequality follows from the definition of a proper matrix score.
\end{proof}

\subsection{Proof of Proposition~\ref{prop:ent_to_div}}

\begin{proposition}[Proposition~\ref{prop:ent_to_div} restated]
    Given a proper matrix score $\mathbf{S}$, it holds
    that $H_{\mathbf{S}}$ is concave.
    Further, if $H_{\mathbf{S}}$ is also differentiable,
    then
    \begin{equation}
        D_{\mathbf{S}} = \operatorname{Div}_{-H_{\mathbf{S}}}.
    \end{equation}
\end{proposition}

\begin{proof}
First, we show concavity of $H_{\mathbf{S}}$:
\begin{equation}
\begin{split}
    & H_{\mathbf{S}} \left( \lambda M_1 + \left( 1 - \lambda \right) M_2 \right) \\
    & = \operatorname{tr} S \left( \lambda M_1 + \left( 1 - \lambda \right) M_2 \right) \lambda M_1 + \operatorname{tr} S \left( \lambda M_1 + \left( 1 - \lambda \right) M_2 \right) \left( 1 - \lambda \right) M_2 \\
    & \geq \operatorname{tr} S \left( M_1 \right) \lambda M_1 + \operatorname{tr} S \left( M_2 \right) \left( 1 - \lambda \right) M_2 \\
    & = \lambda H_{\mathbf{S}} \left( M_1 \right) + \left( 1 - \lambda \right) H_{\mathbf{S}} \left( M_2 \right),
\end{split}
\end{equation}
for any $\lambda \in (0,1)$ and $M_1, M_2 \in \mathbb{H}_d^{\Delta}$.

Now comes the divergence part.
For any $M_1, M_2 \in \mathbb{H}_d^{\Delta}$ it holds
\begin{equation}
\begin{split}
    D_{\mathbf{S}} \left( M_1, M_2 \right) & = \operatorname{tr} \left[ \mathbf{S} \left( M_2 \right) M_1 - \mathbf{S} \left( M_1 \right) M_1 \right] \\
    & = \operatorname{tr} \left[ \mathbf{S} \left( M_2 \right) M_1 \right] - H_\mathbf{S} \left( M_1 \right) \\
    & = \operatorname{tr} \left[ \mathbf{S} \left( M_2 \right) \left( M_1 - M_2 \right) \right] - H_\mathbf{S} \left( M_1 \right) + H_\mathbf{S} \left( M_2 \right). \\
\end{split}
\end{equation}
Since $D_{\mathbf{S}} \left( M_1, M_2 \right) \geq 0$ it follows that $H_\mathbf{S} \left( M_1 \right) \geq \operatorname{tr} \left[ \mathbf{S} \left( M_2 \right) \left( M_1 - M_2 \right) \right] + H_\mathbf{S} \left( M_2 \right)$.
In addition to $-H_{\mathbf{S}}$ being convex, it follows that $\mathbf{S}$ is a subgradient of $-H_{\mathbf{S}}$, i.e., $\mathbf{S} \left( M_1 \right) \in \partial (-H_{\mathbf{S}})(M_1)$.
Since $-H_{\mathbf{S}}$ is differentiable, any subgradient is equal to the gradient, i.e., $\mathbf{S} = \nabla H_{\mathbf{S}}$.
Therefore
\begin{equation}
    D_{\mathbf{S}} \left( M_1, M_2 \right) = \operatorname{tr} \left[ \nabla H_{\mathbf{S}} \left( M_2 \right) \left( M_1 - M_2 \right) \right] - H_\mathbf{S} \left( M_1 \right) + H_\mathbf{S} \left( M_2 \right) = \operatorname{Div}_{-H_{\mathbf{S}}} \left( M_1, M_2 \right).
\end{equation}
\end{proof}

\subsection{Proof of Lemma~\ref{prop:ev_risk=ev_unc}}

\begin{lemma}[Lemma~\ref{prop:ev_risk=ev_unc} restated]
Given a proper matrix score $\mathbf{S}$ and a density matrix predictor $d$, which is matrix calibrated for a joint distribution $\mathbb{P}_{X\mathbf{Y}}$, it holds that
\begin{equation}
\begin{split}
    \underbrace{\mathbb{E} \left[ H_{\mathbf{S}} \left( d \left( X \right) \right) \right]}_{\text{Expected Uncertainty}} = \underbrace{\mathcal{R}_{\mathbf{S}} \!\left( d \right)}_{\text{Risk}}.    
\end{split}
\end{equation}
\end{lemma}

As a further, unstated minor assumption, we assume that all integrals used in the proof exist.

\begin{proof}
We omit \emph{almost surely} in the following equations.
Note that from $d$ matrix calibrated follows that $\mathbb{D}_{Y \mid d\left(X\right)} = d\left(X\right)$.
Further,
\begin{equation}
\begin{split}
     \mathbb{E}_{X\mathbf{Y}} \left[ S \left(d \left( X \right), Y \right) \right] & = \mathbb{E}_{X} \left[ \mathbb{E}_{\mathbf{Y}} \left[ S \left( d \left( X \right), Y \right) \mid d \left( X \right) \right] \right] \\
    & = \mathbb{E}_{X} \left[ \mathbb{E}_{\mathbf{Y} \sim \mathbb{P}_{Y \mid d \left( X \right)}} \left[ S \left(d \left( X \right), Y \right) \right] \right] \\
    & = \mathbb{E}_{X} \left[ \int_{\mathcal{Y}_d} S \left(d \left( X \right), y \right) \mathrm{d} \mathbb{P}_{Y \mid d \left( X \right)} \left( y \right) \right] \\
    & = \mathbb{E}_{X} \left[ \int_{\mathcal{Y}_d} y \mathbf{S} \left( d \left( X \right) \right) y^\intercal \mathrm{d} \mathbb{P}_{Y \mid d \left( X \right)} \left( y \right) \right] \\
    & = \mathbb{E}_{X} \left[ \int_{\mathcal{Y}_d} \operatorname{tr} \left( S \left( d \left( X \right) \right) y y^\intercal \right) \mathrm{d} \mathbb{P}_{Y \mid d \left( X \right)} \left( y \right) \right] \\
    & = \mathbb{E}_{X} \left[ \operatorname{tr} \left( S \left( d \left( X \right) \right) \int_{\mathcal{Y}_d} y y^\intercal \mathrm{d} \mathbb{P}_{Y \mid d \left( X \right)} \left( y \right) \right) \right] \\
    & = \mathbb{E}_{X} \left[ \operatorname{tr} \left( S \left( d \left( X \right) \right) \mathbb{D}_{Y \mid d \left( X \right)}
    \right) \right] \\
    & = \mathbb{E}_{X} \left[ \operatorname{tr} \left( S \left( d \left( X \right) \right) d \left( X \right)
    \right) \right] \\
    & = \mathbb{E}_{X} \left[ H_{\mathbf{S}} \left( d \left( X \right) \right) \right], \\
\end{split}
\end{equation}
where the matrix calibration assumption is used in the second-to-last equation.
\end{proof}

\subsection{Proof of Theorem~\ref{th:risk_unc_convergence}}

\begin{theorem}[Theorem~\ref{th:risk_unc_convergence} restated]
Let $d_1, d_2, \dots$ be a sequence of density matrix predictors such that $\lim_{n \to \infty} d_n$ is matrix calibrated.
Then, for a continuous proper matrix score $\mathbf{S}$, it holds
\begin{equation}
    \underbrace{\mathbb{E} \left[ H_{\mathbf{S}} \left( d_n \left( X \right) \right) \right]}_{\text{Expected Uncertainty}} - \underbrace{\mathcal{R}_{\mathbf{S}} \!\left( d_n \right)}_{\text{Risk}} \overset{}{\longrightarrow} 0
\end{equation}
for $n \to \infty$.
\end{theorem}

\begin{proof}
Since $\mathbf{S}$ is continuous, so is $H_{\mathbf{S}}$.
We use $S$ to denote the respective proper score based on $\mathbf{S}$.
It holds
\begin{equation}
\begin{split}
    \lim_{n \to \infty} \mathbb{E}_{X} \left[ H_{\mathbf{S}} \left( d_n \left( X \right) \right) \right] & = \mathbb{E}_{X} \left[ H_{\mathbf{S}} \left( \lim_{n \to \infty} d_n \left( X \right) \right) \right] \\
    & = \mathbb{E}_{X\mathbf{Y}} \left[ S \left( \lim_{n \to \infty} d_n \left( X \right), \mathbf{Y} \right) \right] \\
    & = \lim_{n \to \infty} \mathbb{E}_{X\mathbf{Y}} \left[ S \left( d_n \left( X \right), \mathbf{Y} \right) \right],
\end{split}
\end{equation}
where the second-to-last equality follows from the assumption and Lemma~\ref{prop:ev_risk=ev_unc}.
\end{proof}

\subsection{Proof of Lemma~\ref{th:ms_decomp}}

\begin{lemma}[Lemma~\ref{th:ms_decomp} restated]
Let $\mathbf{S}$ be a proper matrix score and $d$ be a density matrix predictor for a joint distribution $\mathbb{P}_{X\mathbf{Y}}$.
Assuming all integrals are finite, it holds
\begin{equation}
\begin{split}
    \underbrace{\mathcal{R}_S \! \left( d \right)}_{\text{Risk}} & = \underbrace{\operatorname{Cal}_S \!\left( d \right)}_{\text{Calibration}} - \underbrace{\mathbb{E} \left[ D_{\mathbf{S}} \left( \mathbb{D}_{Y \mid d\left(X\right)}, \mathbb{D}_{Y} \right) \right]}_{\text{Sharpness}} + \underbrace{H_{\mathbf{S}} \left( \mathbb{D}_{Y} \right)}_{\text{Noise}},
\end{split}
\end{equation}
with the respective calibration error being defined by
\begin{equation}
    \operatorname{Cal}_S \!\left( d \right) \coloneqq \mathbb{E} \left[ D_{\mathbf{S}} \left( \mathbb{D}_{Y \mid d\left(X\right)}, d\left(X\right) \right) \right].
\end{equation}
\end{lemma}

\begin{proof}
In the following, we use the definitions of the risk, divergence, and entropy associated with $\mathbf{S}$.
\begin{equation}
\begin{split}
& \mathbb{E}_{X \mathbf{Y}} \left[ \mathbf{Y}^\intercal \mathbf{S} \left( d \left( X \right) \right) \mathbf{Y} \right] \\
& = \mathbb{E}_{d \left( X \right)} \left[ \mathbb{E}_{\mathbf{Y}} \left[ \mathbf{Y}^\intercal \mathbf{S} \left( d \left( X \right) \right) \mathbf{Y} \mid d \left( X \right) \right] \right] \\
& = \mathbb{E}_{d \left( X \right)} \left[ \operatorname{tr} \left[ \mathbf{S} \left( d \left( X \right) \right) \mathbb{D}_{Y \mid d \left( X \right)} \right] \right] \\
& = \mathbb{E}_{d \left( X \right)} \left[ D_{\mathbf{S}} \left( \mathbb{D}_{Y \mid d \left( X \right)}, \mathbb{D}_{Y} \right) \right] - \mathbb{E} \left[ H_{\mathbf{S}} \left( \mathbb{D}_{Y \mid d \left( X \right)} \right) \right] \\
& = \mathbb{E}_{d \left( X \right)} \left[ D_{\mathbf{S}} \left( \mathbb{D}_{Y \mid d \left( X \right)}, d \left( X \right) \right) \right] - \mathbb{E} \left[ H_{\mathbf{S}} \left( \mathbb{D}_{Y \mid d \left( X \right)} \right) \right] + H_{\mathbf{S}} \left( \mathbb{D}_Y \right) - H_{\mathbf{S}} \left( \mathbb{D}_Y \right) \\
& = \mathbb{E}_{d \left( X \right)} \left[ D_{\mathbf{S}} \left( \mathbb{D}_{Y \mid d \left( X \right)}, d \left( X \right) \right) \right] - \mathbb{E}_{d \left( X \right)} \left[ D_{\mathbf{S}} \left( \mathbb{D}_{Y \mid d \left( X \right)}, \mathbb{D}_{Y} \right) \right] - H_{\mathbf{S}} \left( \mathbb{D}_Y \right). \\
\end{split}
\end{equation}

\end{proof}

\subsection{Proof of Theorem~\ref{th:cal_risk_optim}}

\begin{theorem}[Theorem~\ref{th:cal_risk_optim} restated]
    Given a proper matrix score $S \colon \mathbb{H}_d^\Delta \times \mathcal{Y}^d \to $, a density matrix predictor $d \colon \mathcal{X} \to \mathbb{H}_d^\Delta$ and an injective function $h \colon \mathbb{H}_d^\Delta \to \mathbb{H}_d^\Delta$, it holds
    \begin{equation}
    \begin{split}
        \underbrace{\mathcal{R}_S \!\left( h \!\circ\! d \right) - \mathcal{R}_S \! \left( d \right)}_{\text{Change in Risk}} = \underbrace{\operatorname{Cal}_S \!\left( h \!\circ\! d \right) - \operatorname{Cal}_S \!\left( d \right)}_{\text{Change in Calibration}}.     
    \end{split}
    \end{equation}
\end{theorem}

\begin{proof}
    According to Lemma~\ref{th:ms_decomp}, the only difference between the risk and the respective calibration error is the sharpness and entropy terms.
    The entropy term is independent of the prediction.
    Therefore, we only need to show that
    \begin{equation}
        \mathbb{E}_{d \left( X \right)} \left[ D_{\mathbf{S}} \left( \mathbb{D}_{Y \mid d \left( X \right)}, \mathbb{D}_{Y} \right) \right] = \mathbb{E}_{d \left( X \right)} \left[ D_{\mathbf{S}} \left( \mathbb{D}_{Y \mid h \left( d \left( X \right) \right)}, \mathbb{D}_{Y} \right) \right].
    \end{equation}
    This follows since from the injectivity of $h$ follows $\mathbb{P}(Y \mid h \left( d \left( X \right) \right)) = \mathbb{P} (Y \mid d \left( X \right))$ from which follows that $\mathbb{D}_{Y \mid h \left( d \left( X \right) \right)} = \mathbb{D}_{Y \mid d \left( X \right)}$.
    This is analogous to \citep{gruber2022better}.
\end{proof}

\subsection{Proof of Proposition~\ref{prop:mat_TS}}

\begin{proposition}[Proposition~\ref{prop:mat_TS} restated]
    The function $h_{\operatorname{TS}} \colon \mathbb{H}_d^\Delta \to \mathbb{H}_d^\Delta$ defined via 
    \begin{equation}
        h_{\operatorname{TS}} \left( M \right) \coloneqq \frac{1}{\operatorname{tr}M^{\alpha}}M^{\alpha} =  \frac{1}{\sum_{i=1}^d \lambda_i^\alpha} \sum_{i=1}^d \lambda_i^\alpha e_i e_i^\intercal
    \end{equation}
    is injective for $\alpha >0$.
\end{proposition}

\begin{proof}
    Note that matrix temperature scaling is a ``stretching'' and ``squeezing'' of the ellipse represented by $M$ along its eigenvectors.
    
    Let $M_1, M_2 \in \mathbb{H}_d^{\Delta}$ such that $M_1 \neq M_2$. 
    Let $M_1 = U_1 \Lambda_1 U_1^\intercal$ and $M_2 = U_2 \Lambda_2 U_2^\intercal$ be their eigenvalue decompositions.
    If $\Lambda_1 = \Lambda_2$, then $U_1 \neq U_2$, from which follows
    \begin{equation}
        h_{\operatorname{TS}} \left( M_1 \right) = U_1 \frac{1}{\operatorname{tr}\Lambda_1^{\alpha}}\Lambda_1^{\alpha} U_1^\intercal \neq U_2 \frac{1}{\operatorname{tr}\Lambda_1^{\alpha}}\Lambda_1^{\alpha} U_2^\intercal = h_{\operatorname{TS}} \left( M_2 \right),
    \end{equation}
    i.e., we cannot stretch and squeeze an ellipse into another when they ``point'' in different directions.
    
    Since from $\frac{1}{\operatorname{tr}\Lambda_1^{\alpha}}\Lambda_1^{\alpha} = \frac{1}{\operatorname{tr}\Lambda_2^{\alpha}}\Lambda_2^{\alpha}$
    follows that $c \Lambda_1 = \Lambda_2$ with $c = \left( \frac{\operatorname{tr} \Lambda_2^\alpha}{\operatorname{tr} \Lambda_1^\alpha}\right)^{\frac{1}{\alpha}}$, it is required that $c=1$ otherwise $\Lambda_1$ or $\Lambda_2$ would not have normalized eigenvalues.
    Reversing the implication, it follows that if $\Lambda_1 \neq \Lambda_2$, then $\frac{1}{\operatorname{tr}\Lambda_1^{\alpha}}\Lambda_1^{\alpha} \neq \frac{1}{\operatorname{tr}\Lambda_2^{\alpha}}\Lambda_2^{\alpha}$.
    From this follows
    \begin{equation}
        h_{\operatorname{TS}} \left( M_1 \right) = U_1 \frac{1}{\operatorname{tr}\Lambda_1^{\alpha}}\Lambda_1^{\alpha} U_1^\intercal \neq U_1 \frac{1}{\operatorname{tr}\Lambda_2^{\alpha}}\Lambda_2^{\alpha} U_1^\intercal = h_{\operatorname{TS}} \left( M_2 \right),
    \end{equation}
    when $U_1=U_2$ and $\Lambda_1 \neq \Lambda_2$.
    Therefore, $h_{\operatorname{TS}}$ is injective.
\end{proof}

\end{document}